\pdfoutput=1
\documentclass{article}
\usepackage[margin=1in]{geometry}
\usepackage[backend=biber,style=alphabetic,natbib=true]{biblatex}
\usepackage{eqparbox}

\usepackage{mathpazo}
\usepackage{multirow}

\usepackage{booktabs}
\usepackage{comment}
\usepackage[T1]{fontenc}
\usepackage[utf8]{inputenc}
\usepackage{mathtools,amssymb,amsthm}
\usepackage{bm}
\usepackage{wasysym}
\usepackage{thmtools}
\usepackage{thm-restate}
\usepackage{graphicx}
\usepackage{subcaption}
\usepackage{floatrow}
\usepackage[disable]{todonotes}
\usepackage{textgreek}

\usepackage[pdfencoding=auto]{hyperref}

\newtheorem{theorem}{Theorem}
\newtheorem{lemma}{Lemma}

\DeclareMathOperator*{\argmin}{arg\,min}

\newcommand{\nrf}{non-robust features}
\newcommand{\inner}[2]{\left\langle {#1}, {#2} \right\rangle}

\newcommand{\D}{\mathcal{D}}
\newcommand{\F}{\mathcal{F}}
\newcommand{\X}{\mathcal{X}}
\newcommand{\Reals}{\mathbb{R}}
\newcommand{\pmo}{\{\pm 1\}}
\newcommand{\sign}{\text{sgn}}

\newcommand{\tr}{\text{tr}}
\newcommand{\dia}{\invdiameter}
\newcommand{\tvgrad}{ \stack{\frac{1}{2} (Axx^\top)_{\dia}}{- Ax} }

\newcommand{\eps}{\varepsilon}
\newcommand{\E}{\mathbb{E}_{x\sim\mathcal{N}(\mu^*, \Sigma^*)}}
\newcommand{\ED}{\mathbb{E}_{(x,y)\sim\D}}

\newcommand{\EDc}[1]{\mathbb{E}_{(x, y)\sim{#1}}}

\newcommand{\stack}[2]{\begin{bmatrix}
	{#1} \\
	{#2}
\end{bmatrix}}

\let\oldmu\mu
\let\oldsig\Sigma
\renewcommand{\mu}{\bm{\oldmu}}
\renewcommand{\Sigma}{\bm{\oldsig}}

\newfloatcommand{capbtabbox}{table}[][\FBwidth]

\title{Adversarial Examples Are Not Bugs, They Are Features}

\newcommand\AND{
    \end{tabular}\hfil\linebreak[4]\hfil
    \begin{tabular}[t]{c}\ignorespaces
}
\author{Andrew Ilyas\footnote{Equal contribution} \\
    MIT \\
  \texttt{ailyas@mit.edu} \\
   \and
   Shibani Santurkar\footnotemark[1] \\
    MIT \\
  \texttt{shibani@mit.edu} \\
  \and 
  Dimitris Tsipras\footnotemark[1] \hfill\null \\
    MIT \\
  \texttt{tsipras@mit.edu} \\ \\
  \AND 
  Logan Engstrom\footnotemark[1] \\
    MIT \\
  \texttt{engstrom@mit.edu} \\
  \and 
  Brandon Tran \\ 
    MIT \\
  \texttt{btran115@mit.edu} \\
  \and 
  Aleksander M\k{a}dry \\
    MIT \\
  \texttt{madry@mit.edu}  
}

\date{}

\begin{document}
\maketitle
\begin{abstract}
Adversarial examples have attracted 
significant attention in machine learning, but the reasons for their existence 
and pervasiveness remain unclear.
We demonstrate that adversarial examples can be directly attributed to the 
presence of {\em non-robust features}: features 
(derived from patterns in the data distribution) that are highly predictive,
 yet brittle and (thus) incomprehensible to humans. 
After capturing these features within a theoretical framework, we 
establish their widespread existence in standard datasets.
Finally, we present a simple setting where we can rigorously tie the phenomena 
we observe in practice to a {\em misalignment} between the (human-specified)
notion of robustness and the inherent geometry of the data.

\end{abstract}

\section{Introduction}
The pervasive brittleness of deep neural
networks~\citep{szegedy2014intriguing,engstrom2019rotation,hendrycks2019benchmarking,athalye2018synthesizing} has attracted 
significant attention in recent years. 
Particularly worrisome is the phenomenon 
of {\em adversarial examples}~\citep{biggio2013evasion,szegedy2014intriguing},
imperceptibly perturbed natural inputs that induce erroneous predictions in state-of-the-art classifiers.
Previous work has proposed a variety of explanations for this
phenomenon, ranging from theoretical models
\citep{schmidt2018adversarially,bubeck2018adversarial} to arguments based on
concentration of measure in high-dimensions
\citep{gilmer2018adversarial,mahloujifar2018curse,shafahi2019are}. These
theories, however, are often unable to fully capture behaviors we observe  
in practice (we discuss this further in Section~\ref{sec:related}).

More broadly, previous work in the field tends to view adversarial examples
as aberrations arising either from the high dimensional nature of the input
space or statistical fluctuations in the training
data~\cite{szegedy2014intriguing,goodfellow2015explaining,gilmer2018adversarial}. 
From this point of view, it is natural to treat adversarial robustness as a
goal that can be disentangled and pursued independently from maximizing
accuracy~\citep{madry2018towards,stutz2019disentangling,suggala2019adversarial},
either through improved standard regularization
methods~\citep{tanay2016boundary} or
pre/post-processing of network
inputs/outputs~\cite{uesato2018adversarial,carlini2017adversarial,he2017adversarial}.

In this work, we propose a new perspective on the phenomenon of adversarial examples. 
In contrast to the previous models, we cast
adversarial vulnerability as a fundamental consequence of the
dominant supervised learning paradigm. Specifically, we claim that:
\begin{center}
    \em Adversarial vulnerability is a direct result of
    our models' sensitivity to well-generalizing features in the data.
\end{center}
Recall that we usually train classifiers to
{\em solely}  maximize (distributional) accuracy. Consequently,
classifiers tend to use {\em any} available signal to do so, even those that look
incomprehensible to humans. After all, the presence of ``a tail'' or
``ears'' is no more natural to a classifier than any other equally
predictive feature.  
In fact, we find that standard ML datasets {\em do}
admit highly predictive yet imperceptible features.
We posit that our models learn to rely on these ``non-robust'' features,
leading to adversarial perturbations that exploit this dependence.\footnote{It is worth emphasizing that while our findings demonstrate that adversarial vulnerability {\em does} arise from non-robust features, they do not preclude the possibility of adversarial vulnerability also arising from other phenomena~\cite{tanay2016boundary,schmidt2018adversarially}. For example,~\citet{nakkiran2019bugs} constructs adversarial examples that do not exploit non-robust features (and hence do not allow one to learn a generalizing model from them). Still, the mere existence of useful non-robust features suffices to establish that without explicitly discouraging models from utilizing these features, adversarial vulnerability will remain an issue.}

Our hypothesis also suggests an explanation for {\em adversarial transferability}: the phenomenon that adversarial perturbations computed for one model often transfer to other,
independently trained models.
Since any two models are likely to learn similar non-robust features, perturbations that manipulate such features will apply to both.
Finally, this perspective establishes adversarial vulnerability as a 
human-centric phenomenon, since, from the standard supervised learning point
of view, non-robust features can be as important as robust ones. It also
suggests that approaches aiming to enhance the interpretability of a given model by enforcing ``priors'' for
its
explanation~\cite{mahendran2015understanding,olah2017feature,smilkov2017smoothgrad} actually hide features that are {\em ``meaningful''} and {\em predictive}
to standard models. As such, producing {\em human}-meaningful explanations that
remain faithful to underlying models cannot be pursued independently from the
training of the models themselves.

To corroborate our theory, we show that it is possible to disentangle robust from non-robust
features in standard image classification datasets. 
Specifically, given any training dataset, we are able to construct:
\begin{enumerate}
    \item {\bf A ``robustified'' version for robust classification
	(Figure~\ref{fig:robustify_diagram})\footnote{The corresponding datasets 
	for CIFAR-10 are publicly available at~\url{http://git.io/adv-datasets}.} .}
	We demonstrate that it is possible to 
    effectively remove non-robust features from a dataset.
    Concretely, we create a training set (semantically similar to the original) on which {\em standard
    training} yields {\em good robust accuracy} on the {\em original, unmodified} test set.
    This finding establishes that adversarial vulnerability is not necessarily
    tied to the standard training framework, but is also a property of the
    dataset.
\item {\bf A ``non-robust'' version for standard classification
    (Figure~\ref{fig:adv_exp_diagram})\footnotemark[2].}  We are also 
    able to construct a training dataset for which the inputs are nearly identical to the originals, 
    but all appear incorrectly labeled. In fact, the inputs in the new training set are associated to their 
    labels only through {\em small adversarial perturbations} (and hence utilize only non-robust features).
    Despite the lack of any predictive human-visible information, training on this dataset
    yields good accuracy on the {\em original, unmodified} test set. 
    This demonstrates that adversarial perturbations can arise from flipping
    features in the data that are useful for classification of correct inputs
    (hence not being purely aberrations).
\end{enumerate}

Finally, we present a concrete classification task where the 
connection between adversarial examples and non-robust features can be studied rigorously.
This task consists of separating Gaussian distributions, and is loosely
based on the model presented in~\citet{tsipras2019robustness}, while 
expanding upon it in a few ways.
First, adversarial vulnerability in our setting can be precisely quantified as
a difference between the intrinsic data geometry and that of the
adversary's perturbation set.
Second, robust training yields a classifier which utilizes a geometry 
corresponding to a combination of these two.
Lastly, the gradients of standard models can be significantly more 
misaligned with the inter-class direction, capturing a phenomenon 
that has been observed in practice in more complex scenarios~\citep{tsipras2019robustness}.
\begin{figure}[ht]
	\centering
	\begin{subfigure}[b]{0.5\textwidth}
	\centering
	\includegraphics[width=1.0\textwidth]{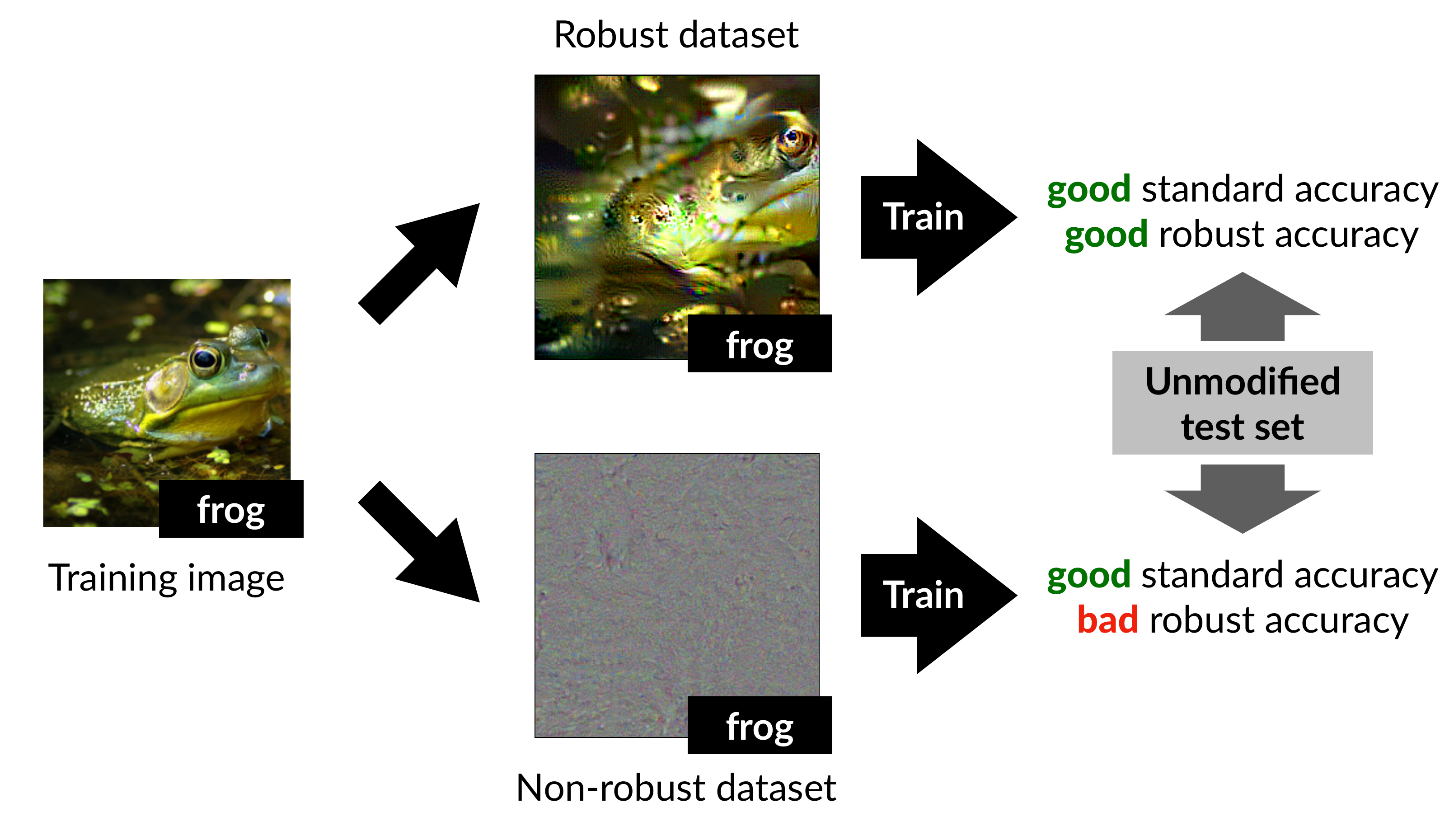}
	\caption{}
	\label{fig:robustify_diagram}
	\end{subfigure}
\hfill
	\begin{subfigure}[b]{0.4\textwidth}
	\centering
	\includegraphics[width=1.0\textwidth]{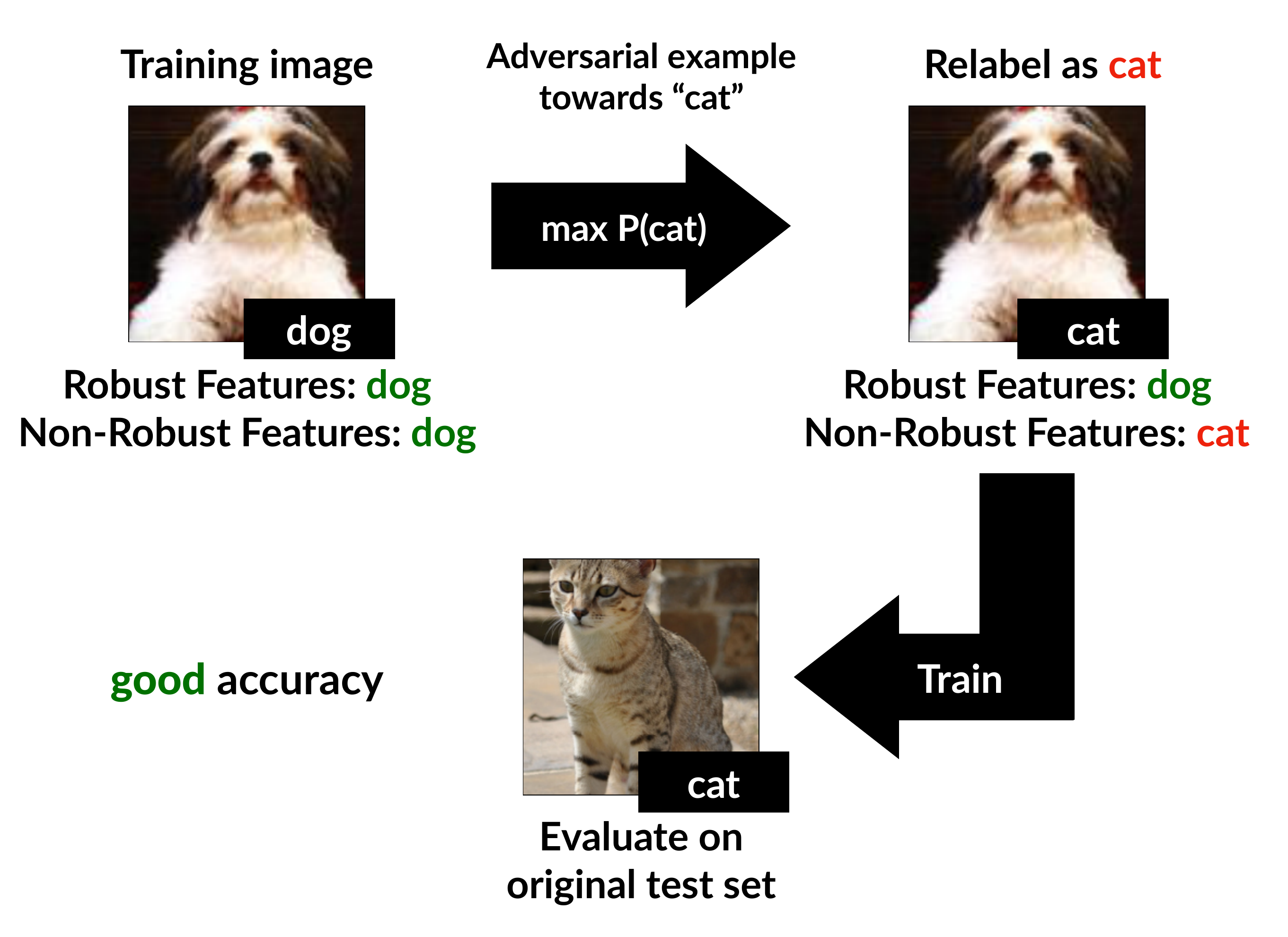}
	\caption{}
	\label{fig:adv_exp_diagram}
	\end{subfigure}
    \caption{A conceptual diagram of the experiments of 
    Section~\ref{sec:finding}. In (a) we disentangle features into
    combinations of robust/non-robust features (Section~\ref{sec:distill}). 
    In (b) we construct a dataset
which appears mislabeled to humans (via adversarial examples) but results
in good accuracy on the original test set (Section~\ref{sec:distill_nrf}).}
\label{fig:distill}
\end{figure}

\section{The Robust Features Model}
\label{sec:formal}
We begin by developing a framework, loosely based on the setting proposed
by~\citet{tsipras2019robustness}, that enables us to rigorously refer to
``robust'' and ``non-robust'' features. In particular, we present a set of
definitions which allow us to formally describe our setup, theoretical
results, and empirical evidence.

\paragraph{Setup.} 
We consider binary classification\footnote{Our framework can
    be straightforwardly adapted though to the multi-class setting.}, where
    input-label pairs $(x, y)\in\X\times\pmo$ are sampled from a (data)
    distribution $\D$; the goal is to 
learn a classifier $C: \X\rightarrow\pmo$ which predicts a label $y$ corresponding 
to a given input $x$.

We define a {\em feature} to be a function mapping from the
input space $\X$ to the real numbers, with the set of all features thus
being $\F = \{f: \X\to\Reals\}$. For convenience, we assume that
the features in $\F$ are shifted/scaled to be mean-zero and unit-variance
(i.e., so that $\ED[f(x)] = 0$ and $\ED[f(x)^2]=1$), in order to make the following definitions
scale-invariant\footnote{This restriction can be straightforwardly
removed by simply shifting/scaling the definitions.}. Note that this formal
definition also captures what we abstractly think of as features (e.g., we
can construct an $f$ that captures how ``furry'' an image is).

\paragraph{Useful, robust, and non-robust features.} We now define the key
concepts required for formulating our framework. To this end, we categorize
features in the following manner:
\begin{itemize}
\item {\bf $\rho$-useful features}: For a given distribution $\D$, we call a
feature $f$ {\em $\rho$-useful} ($\rho > 0$) if it is correlated with the true
label in expectation, that is if
\begin{equation} 
\ED[y \cdot f(x)]  \geq \rho.
\end{equation}
We then define $\rho_\D(f)$ as the largest $\rho$ for which feature $f$ is
$\rho$-useful under distribution $\D$. (Note that if a feature $f$ is negatively
correlated with the
label, then $-f$ is useful instead.) Crucially, a linear classifier trained on
$\rho$-useful features can attain non-trivial generalization performance.
\item {\bf $\gamma$-robustly useful features}: Suppose we have a
    $\rho$-useful feature $f$ ($\rho_\D(f) > 0$). We refer to $f$ as a {\em robust
feature} (formally a $\gamma$-robustly useful feature for $\gamma > 0$) if,
under adversarial perturbation (for some specified set of valid 
perturbations $\Delta$), $f$ remains $\gamma$-useful.
Formally, if we have that
\begin{equation}
    \ED\left[\inf_{\delta\in\Delta(x)} y \cdot f(x+\delta)\right]  \geq \gamma.
\end{equation}
\item {\bf Useful, non-robust features}: A {\em useful, non-robust feature} 
is a feature which is $\rho$-useful for some $\rho$ bounded away from zero,
but is not a $\gamma$-robust feature for any $\gamma \geq 0$. These features
help with classification in the standard setting, but may hinder
accuracy in the adversarial setting, as the correlation with the label can
be flipped.
\end{itemize}

\paragraph{Classification.} In our framework, a classifier $C = (F,
w, b)$ is comprised of a set of features $F\subseteq\F$, a weight
vector $w$, and a scalar bias $b$. For a given input $x$, the classifier
predicts the label $y$ as
$$C(x) = \sign\left(b + \sum_{f\in F} w_f \cdot f(x)\right).$$
For convenience, we denote the set of features learned by a classifier $C$
as $F_C$.

\paragraph{Standard Training.} Training a classifier is performed by
minimizing a loss function (via {\em empirical risk minimization} (ERM)) that
decreases with the correlation between the weighted combination of the features
and the label. The simplest example of such a loss is
\footnote{Just as for the other parts of this model, we use this loss
    for simplicity only---it is straightforward to generalize to more
    practical loss function such as logistic or hinge loss.}
\begin{equation}
    \ED\left[\mathcal{L}_\theta(x, y)\right] = -\ED\left[y\cdot
    \left(b + \sum_{f\in F} w_f\cdot f(x)\right)\right].
\end{equation}
When minimizing classification loss, {\em no distinction} exists between
robust and non-robust features: the only distinguishing factor of a feature
is its $\rho$-usefulness. Furthermore, the classifier will utilize {\em
any} $\rho$-useful feature in $F$ to decrease the loss of the classifier.

\paragraph{Robust training.} In the presence of an {\em adversary}, any useful but
non-robust features can be made {\em anti-correlated} with the true label, 
leading to adversarial vulnerability. Therefore, ERM is no longer
sufficient to train
classifiers that are robust, and we 
need to explicitly account for the effect of the adversary on the
classifier. To do so, we use an {\em adversarial} loss function that can discern
between robust and non-robust features \citep{madry2018towards}:
\begin{equation}
    \ED\left[\max_{\delta\in\Delta(x)}\mathcal{L}_\theta(x+\delta, y)\right],
\end{equation}
for an appropriately defined set of perturbations $\Delta$. Since the
adversary can exploit non-robust features to degrade classification
accuracy, minimizing this adversarial loss (as in adversarial
training~\citep{goodfellow2015explaining,madry2018towards}) can be viewed
as explicitly preventing the classifier from learning a useful but
non-robust combination of features.

\paragraph{Remark.} We want to note that even though the framework above enables
us to formally describe and predict the outcome of our experiments, it does not
necessarily capture the
notion of non-robust features exactly as we intuitively might think of them. For
instance, in principle, our theoretical framework would allow for useful
non-robust features to arise as combinations of useful robust features and
useless non-robust features~\cite{goh2019discussion}. These types of
constructions, however, are actually precluded by our experimental results (in
particular, the classifiers trained in Section~\ref{sec:finding} would not
generalize).  This shows that our experimental findings capture a stronger, more
fine-grained statement than our formal definitions are able to express. We view
bridging this gap as an interesting direction for future work.

\section{Finding Robust (and Non-Robust) Features}
\label{sec:finding}
The central premise of our proposed framework is that there exist both robust and
{\nrf} that constitute useful signals for standard classification. We now
provide evidence in support of this hypothesis by disentangling these two sets
of features.

\newcommand{\drobust}{\widehat{\mathcal{D}}_{R}}
\newcommand{\dnonrobust}{\widehat{\mathcal{D}}_{NR}}
\newcommand{\drand}{\widehat{\mathcal{D}}_{rand}}
\newcommand{\ddet}{\widehat{\mathcal{D}}_{det}}

On one hand, we will construct a ``robustified'' dataset, consisting of samples that
primarily contain robust features. Using such a dataset, we are able
to train robust classifiers (with respect to the standard test set) using
standard (i.e., non-robust) training.
This demonstrates that robustness can arise by {\em removing} certain features from
the dataset (as, overall, the new dataset contains less
information about the original training set).
Moreover, it provides evidence that adversarial vulnerability is caused by
{\nrf} and is not inherently tied to the standard training framework.

On the other hand, we will construct datasets where the
input-label association is based purely on non-robust features (and thus
the corresponding dataset appears {\em completely} mislabeled
to humans). We show that this dataset suffices to train a classifier with good
performance on the standard test set.  This indicates that natural models use
{\em non-robust features} to make predictions, even in the presence of robust
features. These features {\em alone} are actually sufficient for non-trivial
generalizations performance on natural images, which indicates that they are
indeed valuable features, rather than artifacts of finite-sample overfitting.

A conceptual description of these experiments can be found in
Figure~\ref{fig:distill}.

\begin{figure}[!ht]
	\centering
	\begin{subfigure}[b]{0.45\textwidth}
	\centering
	\includegraphics[width=1.0\textwidth]{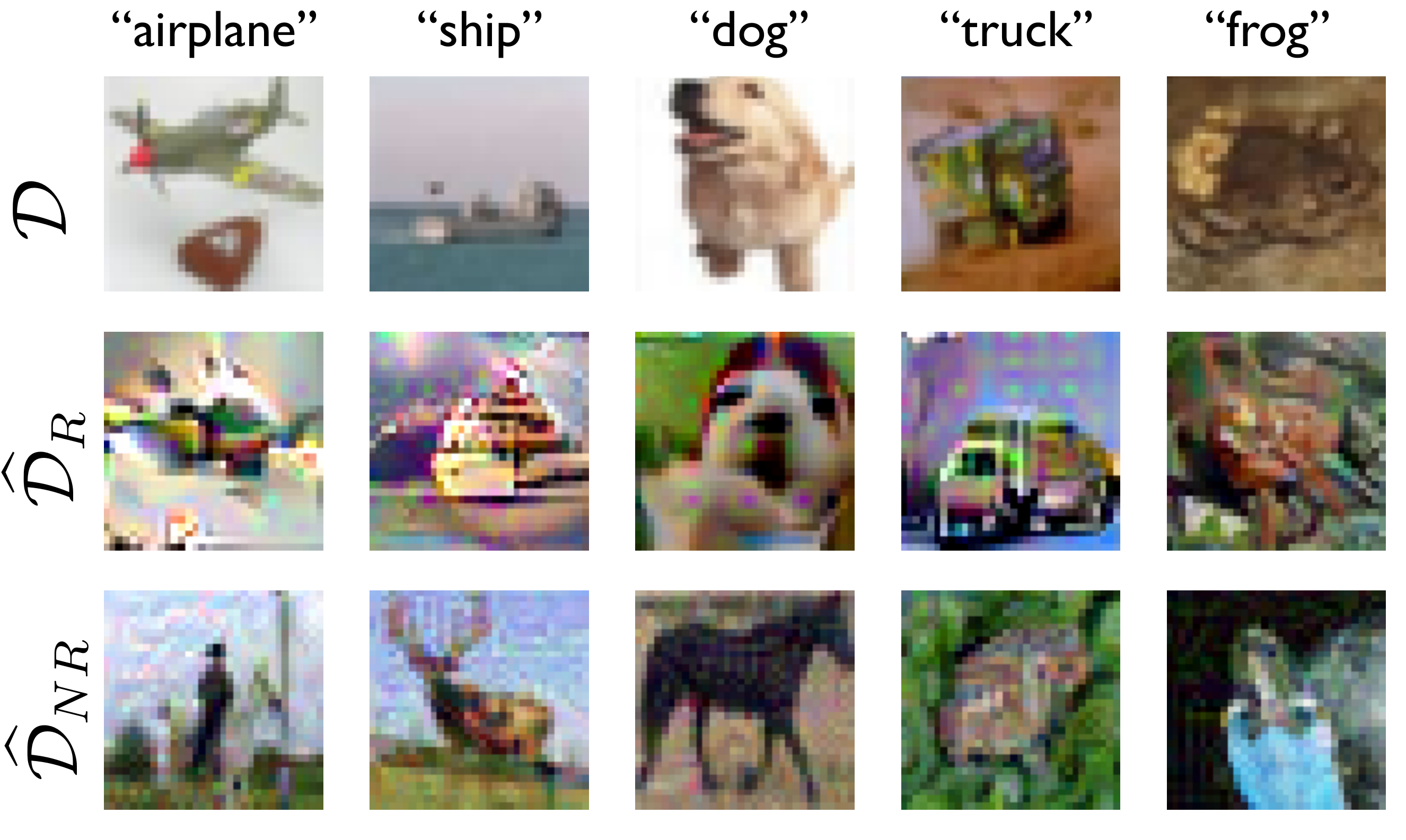}
	\vfill\null
	\caption{}
	\label{fig:robust_inputs}
	\end{subfigure}
\hfill
	\begin{subfigure}[b]{0.5\textwidth}
	\centering
	\includegraphics[width=\textwidth]{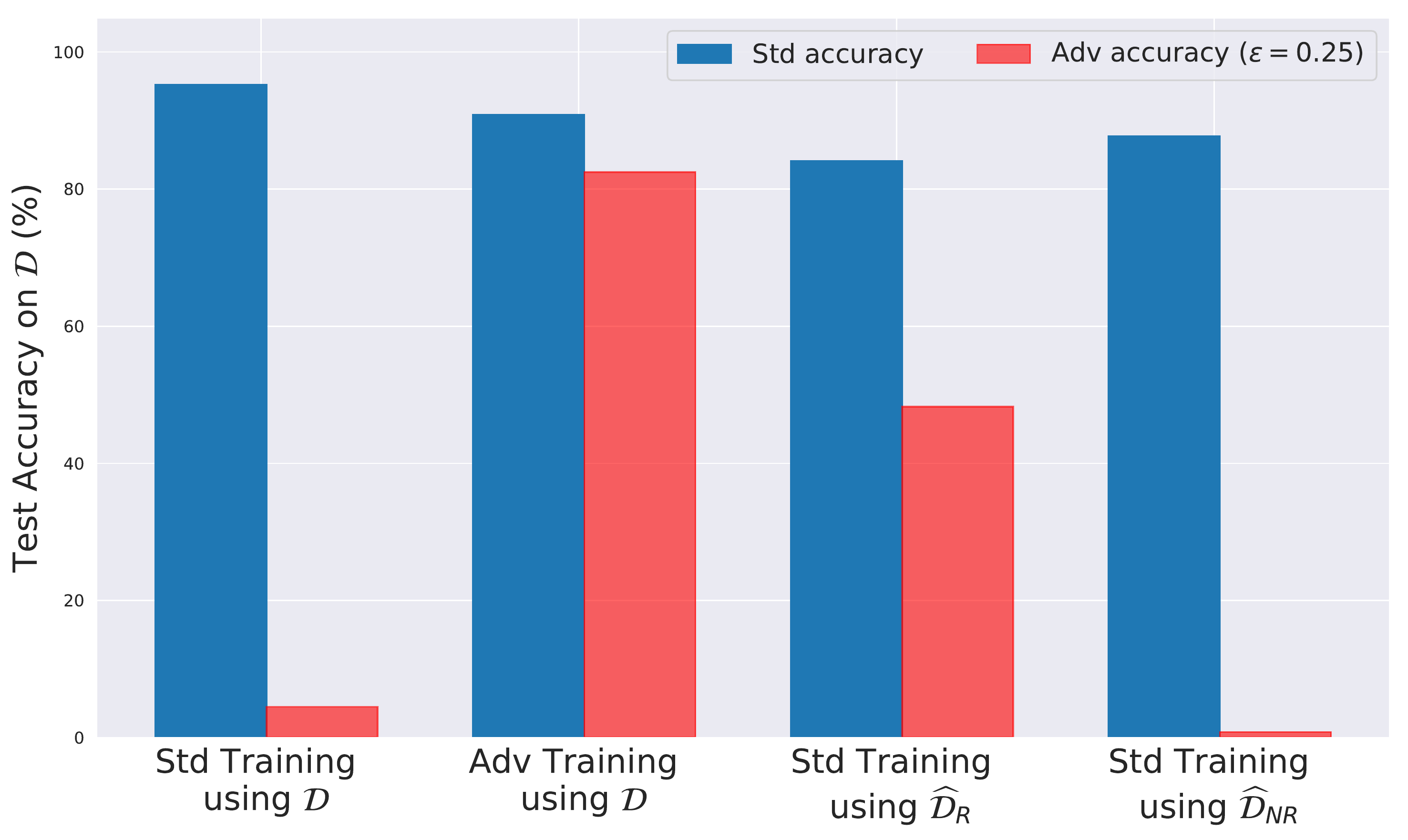}
	\caption{}
	\label{fig:robustify_cifar}
	\end{subfigure}
    \caption{
	{\bf Left}: Random samples from our variants of the
        CIFAR-10~\cite{krizhevsky2009learning} training set:
	the original training set; 
	the {\em robust training set} $\drobust$, restricted to features used by a
	robust model; and
	the {\em non-robust training set} $\dnonrobust$, restricted to
	features relevant to a standard model (labels appear incorrect to humans).
	{\bf Right}: Standard and robust accuracy on the CIFAR-10
	test set ($\D$) for models trained with:
	(i) standard training (on $\D$)  ;
	(ii) standard training on $\dnonrobust$;
	(iii) adversarial training (on $\D$); and
	(iv) standard training on $\drobust$.
        Models trained on $\drobust$ and $\dnonrobust$ reflect
        the original models used to create them: notably,
        standard training on $\drobust$
        yields nontrivial robust accuracy. Results for
Restricted-ImageNet~\citep{tsipras2019robustness} are in
~\ref{app:imagenet} Figure~\ref{fig:distill_imagenet}.}
\label{fig:distill_results}
\end{figure}

\subsection{Disentangling robust and non-robust features}
\label{sec:distill}
Recall that the
features a classifier learns to rely on are based purely
on how useful these features are for (standard) generalization. Thus, under our conceptual framework,
if we can ensure that only robust features are useful, standard training
should result in a robust classifier.
Unfortunately, we cannot directly manipulate the features of very complex,
high-dimensional datasets.
Instead, we will leverage a robust model and modify
our dataset to contain only the features that are relevant to that model.

In terms of our formal framework (Section~\ref{sec:formal}), given a {\em
robust} (i.e., adversarially trained~\cite{madry2018towards}) model $C$ we aim to construct a
distribution $\drobust$ which satisfies:
\begin{equation}
    \label{eq:robustify_cases}
    \mathbb{E}_{(x,y)\sim\drobust}\left[f(x)\cdot y\right] = \begin{cases}
	\ED\left[f(x)\cdot y\right] &\text{if } f \in F_C \\
	0 &\text{otherwise},
    \end{cases}
\end{equation}
where $F_C$ again represents the set of features utilized by $C$.
Conceptually, we want features used by $C$ to be as useful as they were on the
original distribution $\D$ while ensuring that the rest of the features are not
useful under $\dnonrobust$. 

We will construct a training set for $\drobust$ via a one-to-one mapping $x
\mapsto x_r$ from the
original training set for $\D$. 
In the case of a deep neural network, $F_C$
corresponds to exactly the set of activations in the penultimate layer
(since these correspond to inputs to a linear classifier).
To ensure that features used by the model are equally useful under both training
sets, we (approximately) enforce all features in $F_C$ to have similar values for
both $x$ and $x_r$ through the following optimization:
\begin{equation}
\label{eq:robustify}
\min_{x_r} \|g(x_r) - g(x)\|_2,
\end{equation}
where $x$ is the original input and $g$ is the mapping from $x$ to the
representation layer.
We optimize this objective using gradient descent in
input space\footnote{We follow~\cite{madry2018towards} and normalize gradient steps
    during this optimization. Experimental details are provided in
Appendix~\ref{app:experimental_setup}.}.

Since we don't have access to features outside
$F_C$, there is no way to ensure that the expectation
in~\eqref{eq:robustify_cases} is zero for all $f\not\in F_C$. To
approximate this condition, we choose the starting point of gradient descent
for the optimization in~\eqref{eq:robustify}
to be an input $x_0$ which is drawn from $\D$ independently of the label of $x$
(we also explore sampling $x_0$ from noise in
Appendix~\ref{app:from_noise}).
This choice ensures that any feature present in that input will not be useful
since they are not correlated with the label in expectation over $x_0$.
The underlying assumption here is that, when performing the optimization
in~\eqref{eq:robustify}, features that are not being directly optimized (i.e.,
features outside $F_C$) are not affected.
We provide pseudocode for the construction in
Figure~\ref{alg:robustification} (Appendix~\ref{app:experimental_setup}).

Given the new training set for $\drobust$ (a few random samples are
visualized in Figure~\ref{fig:robust_inputs}), we
train a classifier using standard (non-robust) training. We then test this
classifier on the original test set (i.e. $\D$). The results
(Figure~\ref{fig:robustify_cifar}) indicate that the classifier learned
using the new dataset attains good accuracy in {\em both standard and
adversarial settings}
\footnote{In an attempt to explain the gap in
accuracy between the model trained on $\drobust$ and the original
robust classifier $C$, we test distributional shift, by reporting results
on the ``robustified'' test set in
Appendix~\ref{app:robust_test}.}
\footnote{In order to gain more confidence in the robustness of the resulting
    model, we attempt several diverse attacks in Appendix~\ref{app:attacks}.}.

As a control, we repeat this methodology using a standard (non-robust) model for $C$ in
our construction of the dataset. Sample images from the resulting
``non-robust dataset'' $\dnonrobust$ are shown in Figure~\ref{fig:robust_inputs}---they
tend to resemble more the source image of the optimization $x_0$ than the target
image $x$. We find that training on this dataset leads to good standard
accuracy, yet yields almost no robustness (Figure~\ref{fig:robustify_cifar}). We
also verify that this procedure is not simply a matter of encoding the
weights of the original model---we get the same results for both $\drobust$
and $\dnonrobust$ if we train with different architectures than that of the
original models.

Overall, our findings corroborate the hypothesis that adversarial examples
can arise from (non-robust) features of the data itself. By
filtering out \nrf \ from the dataset (e.g. by restricting the
set of available features to those used by a robust model), one can train a
significantly more robust model using {\em standard training}.

\subsection{Non-robust features suffice for standard classification}
\label{sec:distill_nrf}
The results of the previous section show that by restricting the dataset to
only contain features that are used by a robust model, standard training results
in classifiers that are significantly more robust. This suggests that when training on the standard
dataset, non-robust features take on a large role in the resulting learned
classifier. Here we set out to show that this role is not merely incidental or
due to finite-sample overfitting. In particular, we demonstrate that
non-robust features {\em alone} suffice for standard generalization---
i.e., a model trained solely on non-robust features can perform well on
the {\em standard} test set. 

To show this, we construct a dataset where the only features that
are useful for classification are {\em non-robust} features (or in terms of
our formal model from Section \ref{sec:formal}, all features $f$ that are $\rho$-useful are
non-robust).
To accomplish this, we modify each input-label pair $(x, y)$ as follows.
We select a target class $t$ either (a) uniformly at random among
classes (hence features become uncorrelated with the labels) or (b)
deterministically according to the source class (e.g. using a fixed permutation
of labels).
Then, we add a small adversarial perturbation to $x$ in order to ensure it is
classified as $t$ by a standard model. Formally:
\begin{equation}
x_{adv} = \argmin_{\|x' - x\| \leq \eps}\ L_C(x', t),
\label{eqn:adv}
\end{equation}
where $L_C$ is the loss under a standard (non-robust) classifier $C$ and $\eps$
is a small constant. The resulting inputs are nearly indistinguishable
from the originals (Appendix~\ref{app:omitted_figures}
Figure~\ref{fig:distill_nonrobust})---to a human observer,
it thus appears that the label $t$ assigned to the modified input is simply incorrect.
The resulting input-label pairs $(x_{adv},t)$ make up the new
training set (pseudocode in Appendix~\ref{app:experimental_setup} Figure~\ref{alg:nrf}).

Now, since $\|x_{adv}-x\|$ is small, by definition the robust features
of $x_{adv}$ are still correlated with class $y$ (and not $t$) in expectation
over the dataset. After all, humans still recognize the original class.
On the other hand, since every $x_{adv}$ is strongly
classified as $t$ by a standard classifier, it must be that some of the
non-robust features are now strongly correlated with $t$ (in expectation).

In the case where $t$ is chosen at random, the robust features
are originally uncorrelated with the label $t$ (in expectation), and after the
adversarial perturbation can be only slightly correlated (hence being
significantly less useful for classification than before)
\footnote{\citet{goh2019leakage} provides an approach to quantifying this ``robust feature
leakage'' and finds that one can obtain a (small) amount of test accuracy by
leveraging robust feature leakage on $\drand$.}.
Formally, we aim to construct a dataset $\drand$ where
\footnote{Note that the optimization procedure we describe aims to merely
{\em approximate} this condition, where we once again use trained models to simulate
access to robust and non-robust features.}
:
\begin{equation}
    \label{eq:t_rand}
    \EDc{\drand}\left[y\cdot f(x)\right] \begin{cases}
	> 0 &\text{if } f\text{ non-robustly useful under }\D, \\
	\simeq 0 &\text{otherwise.}
    \end{cases}
\end{equation}

In contrast, when $t$ is chosen deterministically based on $y$, the robust
features actually point {\em away} from the assigned label $t$.
In particular, all of the inputs labeled with class $t$ exhibit {\em
non-robust features} correlated with $t$, but robust features correlated
with the original class $y$.
Thus, robust features on the original training set provide significant
predictive power on the training set, but will actually hurt generalization
on the standard test set. Viewing this case again using the formal model,
our goal is to construct $\ddet$ such that 
\begin{equation}
    \label{eq:t_det}
    \EDc{\ddet}\left[y\cdot f(x)\right] \begin{cases}
	> 0 &\text{if } f\text{ non-robustly useful under }\D, \\
    < 0 &\text{if } f\text{ robustly useful under }\D \\
    \in \mathbb{R} &\text{otherwise ($f$ not useful under
    $\D$)\footnotemark}
    
    \end{cases}
\end{equation}
\footnotetext{
    Note that regardless how useful a feature is on $\ddet$, since it is
    useless on $\D$ it cannot provide any generalization benefit on the
    unaltered test set.
}

We find that standard training on these datasets
actually generalizes to the {\em original}
test set, as shown in
Table~\ref{tab:adv_next}).
This indicates that non-robust
features are indeed useful for classification in the standard setting.
Remarkably, even training on $\ddet$ (where all
the robust features are correlated with the wrong class),
results in a well-generalizing classifier.
This indicates that non-robust features can be picked up by models during
standard training,
even in the presence of {\em robust features} that are predictive
\footnote{Additional results and analysis (e.g. training curves,
    generating $\drand$ and $\ddet$ with a robust model, etc.) are in
App. ~\ref{app:erm_relabeled} and~\ref{app:accuracy_curves}}\footnote{We
also show that the models trained on $\drand$ and $\ddet$ generalize to
CIFAR-10.1~\cite{recht2018cifar10} in Appendix~\ref{app:cifar101}.}.

\begin{minipage}{0.45\textwidth}
    
	\includegraphics[width=\textwidth]{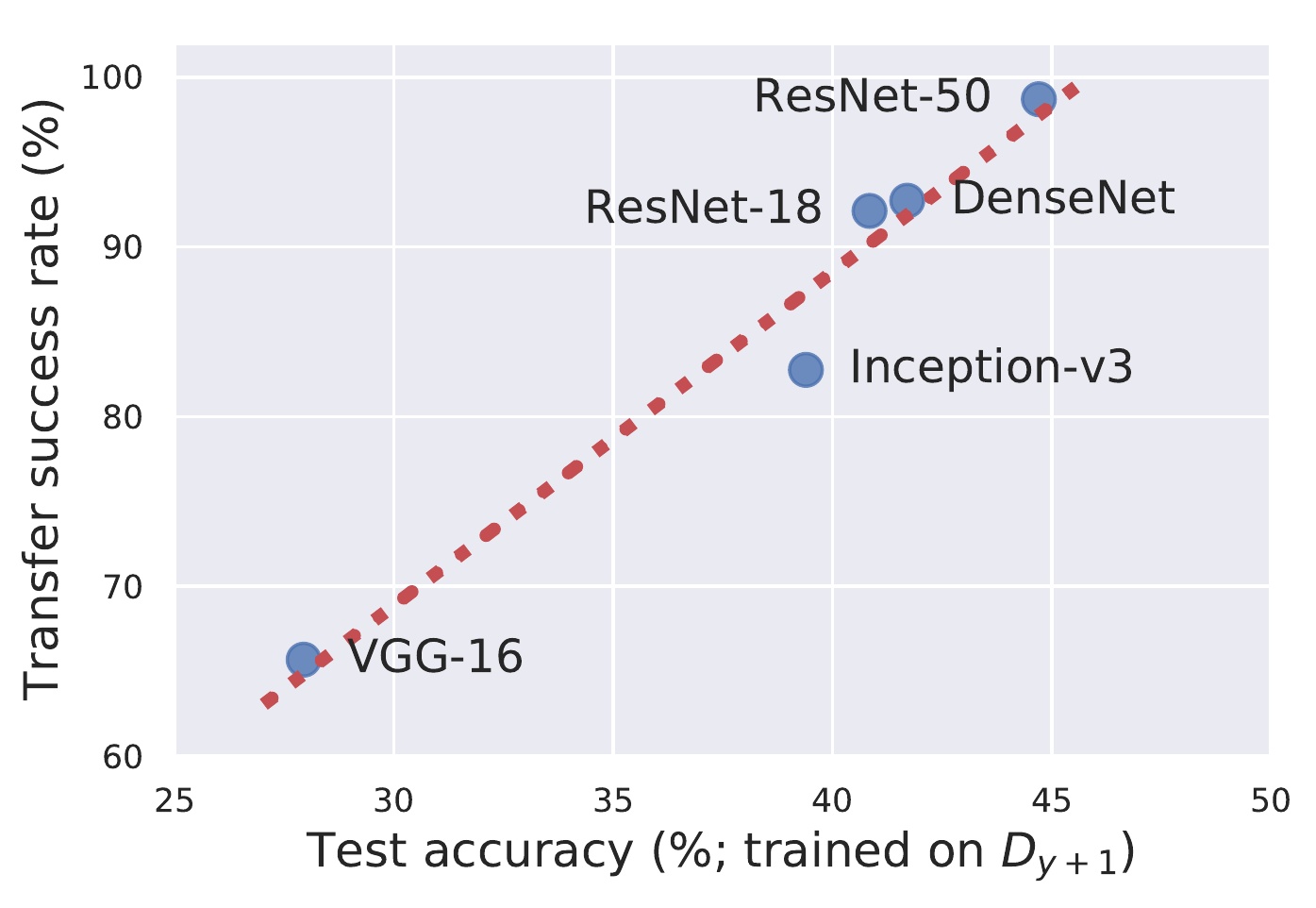}
	\captionof{figure}{Transfer rate of adversarial examples from a ResNet-50
	to different architectures alongside test set performance of
	these architecture when trained on the dataset 
	generated in Section~\ref{sec:distill_nrf}. Architectures more susceptible to transfer
	attacks also performed better on the standard test set supporting our
	hypothesis that adversarial transferability arises from utilizing similar
	{\em non-robust features}.}
     \label{fig:transfer} 
\end{minipage} \hfill
\begin{minipage}{0.45\textwidth}
	\vspace*{1em}
    \begin{center}
    {\renewcommand{\arraystretch}{1.2}
    \begin{tabular}{ccc}
	    \toprule
	    \multirow{2}{*}{{\bf Source Dataset}} & \multicolumn{2}{c}{{\bf Dataset}} \\ \cmidrule{2-3}
	    & CIFAR-10 & ImageNet$_R$ \\
	    \midrule
	    $\D$ &  95.3\% &  96.6\% \\ \midrule
	    $\drand$ & 63.3\% & 87.9\% \\
	    $\ddet$ & 43.7\% & 64.4\%  \\
	    \bottomrule \\
    \end{tabular}
}
\vspace*{2.5em}
\captionof{table}{Test accuracy (on $\D$) of classifiers trained on the
    $\D$, $\drand$, and $\ddet$ training sets created using a standard 
    (non-robust) model. For both $\drand$ and
    $\ddet$, only non-robust features correspond to useful features on both
    the train set and $\D$. These datasets are constructed using
    adversarial perturbations of $x$ towards a class $t$ (random for
    $\drand$ and deterministic for $\ddet$); 
    the resulting images are relabeled as $t$.}
    \label{tab:adv_next} 
\end{center}
\end{minipage}

\subsection{Transferability can arise from \nrf}
One of the most intriguing
properties of adversarial examples is that they {\em transfer} across models
with different architectures and independently sampled
training sets~\citep{szegedy2014intriguing,papernot2016transferability,charles2019geometric}.
Here, we show that this phenomenon can in fact be viewed as
a natural consequence of the existence of non-robust features.
Recall that, according to our main thesis, adversarial examples can arise as a
result of perturbing well-generalizing, yet brittle features.
Given that such features are inherent to the data distribution, different
classifiers trained on independent samples from that distribution are likely to
utilize similar non-robust features.
Consequently, an adversarial example constructed by exploiting the non-robust
features learned by one classifier will transfer to any other classifier
utilizing these features in a similar manner.

In order to illustrate and corroborate this hypothesis, we train five
different architectures on the dataset generated in Section~\ref{sec:distill_nrf}
(adversarial examples with deterministic labels)
for a standard ResNet-50~\cite{he2016deep}.
Our hypothesis would suggest that
architectures which learn better from this training set (in
terms of performance on the standard test set) are more likely to learn 
similar non-robust
features to the original classifier.
Indeed, we find that the test accuracy of each
architecture is predictive of how often adversarial examples transfer from the
original model to standard classifiers with that architecture
(Figure~\ref{fig:transfer}).
In a similar vein, \citet{nakkiran2019bugs} constructs a set of adversarial
perturbations that is explicitly non-transferable and finds that these
perturbations cannot be used to learn a good classifier.
These findings thus corroborate our hypothesis that adversarial
transferability arises when models learn similar brittle features of the
underlying dataset.

\section{A Theoretical Framework for Studying (Non)-Robust Features}
\label{sec:theory}
The experiments from the previous section demonstrate that the conceptual
framework of robust and non-robust features is strongly predictive of the
empirical behavior of state-of-the-art models on real-world datasets.
In order to further strengthen our understanding of the phenomenon, we instantiate
the framework in a concrete setting that allows us to theoretically study
various properties of the corresponding model.
Our model is similar to that of~\citet{tsipras2019robustness} in the sense
that it contains a dichotomy between robust and non-robust features, but 
extends upon it in a number of ways:
\begin{enumerate}
    \item The adversarial vulnerability can be explicitly expressed as a
        difference between the inherent data metric and the $\ell_2$
	metric.
    \item Robust learning corresponds exactly to learning a combination of these
        two metrics.
    \item The gradients of adversarially trained models align better with the 
        adversary's metric.
\end{enumerate}

\paragraph{Setup.} We study a simple problem of {\em maximum likelihood
classification} between two Gaussian distributions. In particular, given samples
$(x, y)$ sampled from $\D$ according to
\begin{equation}
    y \stackrel{\text{u.a.r.}}{\sim} \{-1, +1\}, \qquad x \sim
\mathcal{N}(y\cdot\mu_*, \Sigma_*),
\end{equation}
our goal is to learn parameters $\Theta = (\mu, \Sigma)$ such that 
\begin{equation}
\label{eq:mlc_prob}
\Theta = \arg\min_{\mu, \Sigma} 
\ED \left[\ell(x;y\cdot \mu,\Sigma) \right],
\end{equation}
where $\ell(x;\mu,\Sigma)$ represents the Gaussian negative log-likelihood (NLL)
function. Intuitively, we find the parameters $\mu, \Sigma$ which maximize the
likelihood of the sampled data under the given model. Classification under this
model can be accomplished via likelihood test: given an unlabeled sample $x$, we
predict $y$ as
$$y = \arg\max_y \ell(x;y\cdot\mu, \Sigma) = \text{sign}\left(
x^\top \Sigma^{-1} \mu
\right).$$
\noindent In turn, the {\em robust analogue} of this problem arises from
replacing $\ell(x;y\cdot\mu,\Sigma)$ with the NLL under adversarial
perturbation.
The resulting robust parameters $\Theta_r$ can be written as
\begin{align}
\label{eq:robust_mlc}
\Theta_r = \arg\min_{\mu,\Sigma}
\ED
\left[
    \max_{\|\delta\|_2\leq \varepsilon} \ell(x + \delta; y\cdot\mu, \Sigma)
\right],
\end{align}
A detailed analysis of this setting is in Appendix~\ref{app:proof}---here
we present a high-level overview of the results. 

\paragraph{(1) Vulnerability from metric misalignment (non-robust features).}
Note that in this model, one can
rigorously make reference to an {\em inner product} (and thus a metric) induced
by the features. In particular, one can view the learned parameters of a
Gaussian $\Theta = (\mu, \Sigma)$ as defining an inner product over the input
space given by $\inner{x}{y}_\Theta = (x-\mu)^\top\Sigma^{-1}(y-\mu)$. This in
turn induces the Mahalanobis distance, which represents how a change in the
input affects the features learned by the classifier.
This metric is not necessarily aligned with the metric in which the adversary is
constrained, the $\ell_2$-norm.
Actually, we show that adversarial vulnerability arises exactly as a {\em misalignment}
of these two metrics.

\begin{restatable}[Adversarial vulnerability from misalignment]{theorem}{vulnerability}
    \label{thm:0}
    Consider an adversary whose perturbation is determined by the
    ``Lagrangian penalty'' form of~\eqref{eq:robust_mlc}, i.e.
    $$ \max_\delta \ell(x + \delta; y\cdot\mu, \Sigma) - C\cdot \|\delta\|_2,$$
    where $C \geq \frac{1}{\sigma_{min}(\Sigma_*)}$ is a constant trading
    off NLL minimization and the adversarial constraint\footnote{The
    constraint on $C$ is to ensure the problem is concave.}. 
    Then, the adversarial loss $\mathcal{L}_{adv}$ incurred by the
    non-robustly learned $(\mu, \Sigma)$ is given by:
    \begin{align*}
	\mathcal{L}_{adv}(\Theta) - \mathcal{L}(\Theta) &= 
	\tr\left[\left(I + \left(C \cdot \Sigma_* -
	    I\right)^{-1}\right)^2 \right] - d, 
    \end{align*}
    and, for a fixed $\tr(\Sigma_*) = k$ the above is minimized by
    $\Sigma_* = \frac{k}{d}\bm{I}$.
\end{restatable}

In fact, note that such a misalignment corresponds precisely to the existence of
{\em non-robust features}, as it indicates that ``small'' changes in the
adversary's metric along certain directions can cause large changes 
under the data-dependent notion of distance established by the parameters. 
This is illustrated in Figure~\ref{fig:robust_effect}, where misalignment 
in the feature-induced metric is responsible for the presence of a non-robust 
feature in the corresponding classification problem.

\paragraph{(2) Robust Learning.}  
The optimal (non-robust) maximum likelihood estimate is $\Theta =
\Theta^*$, and thus the vulnerability for the standard MLE estimate is governed
entirely by the true data distribution.
The following theorem characterizes the behaviour of the
learned parameters in the robust problem. \footnote{Note: as discussed in
Appendix~\ref{sec:real_objective}, we study a slight relaxation
of~\eqref{eq:robust_mlc} that approaches exactness exponentially fast as
$d\rightarrow \infty$}.
In fact, we can prove (Section~\ref{sec:danskin_valid}) that performing
(sub)gradient descent on the inner maximization (also known as {\em adversarial
training}~\citep{goodfellow2015explaining,madry2018towards})
yields exactly $\Theta_r$.  We find that as the perturbation
budget $\eps$ is increased, the metric induced by the learned features {\em
mixes} $\ell_2$ and the metric induced by the features.
\begin{restatable}[Robustly Learned Parameters]{theorem}{parameters}
    \label{thm:1}
    Just as in the non-robust case, $\mu_r = \mu^*$, i.e. the true mean is
    learned. For the robust covariance $\Sigma_r$, 
    there exists an $\eps_0 > 0$, such that for any $\eps \in [0, \eps_0)$,
    $$\Sigma_r = \frac{1}{2}\Sigma_* + \frac{1}{\lambda}\cdot \bm{I} +
    \sqrt{\frac{1}{\lambda}\cdot \Sigma_*
    + \frac{1}{4}\Sigma_*^2},
    \qquad \text{ where } \qquad \Omega\left(\frac{1+\eps^{1/2}}{\eps^{1/2} +
	\eps^{3/2}}\right) \leq \lambda \leq
	O\left(\frac{1+\eps^{1/2}}{\eps^{1/2}}\right).$$
\end{restatable}

The effect of robust optimization under an $\ell_2$-constrained adversary is
visualized in Figure~\ref{fig:robust_effect}. As $\epsilon$
grows, the learned covariance becomes more aligned with identity. 
For instance, we can see that the classifier learns to be less sensitive
in certain directions, despite their usefulness for natural classification. 

\begin{figure}[b]
    \begin{center}
	\includegraphics[width=0.24\textwidth]{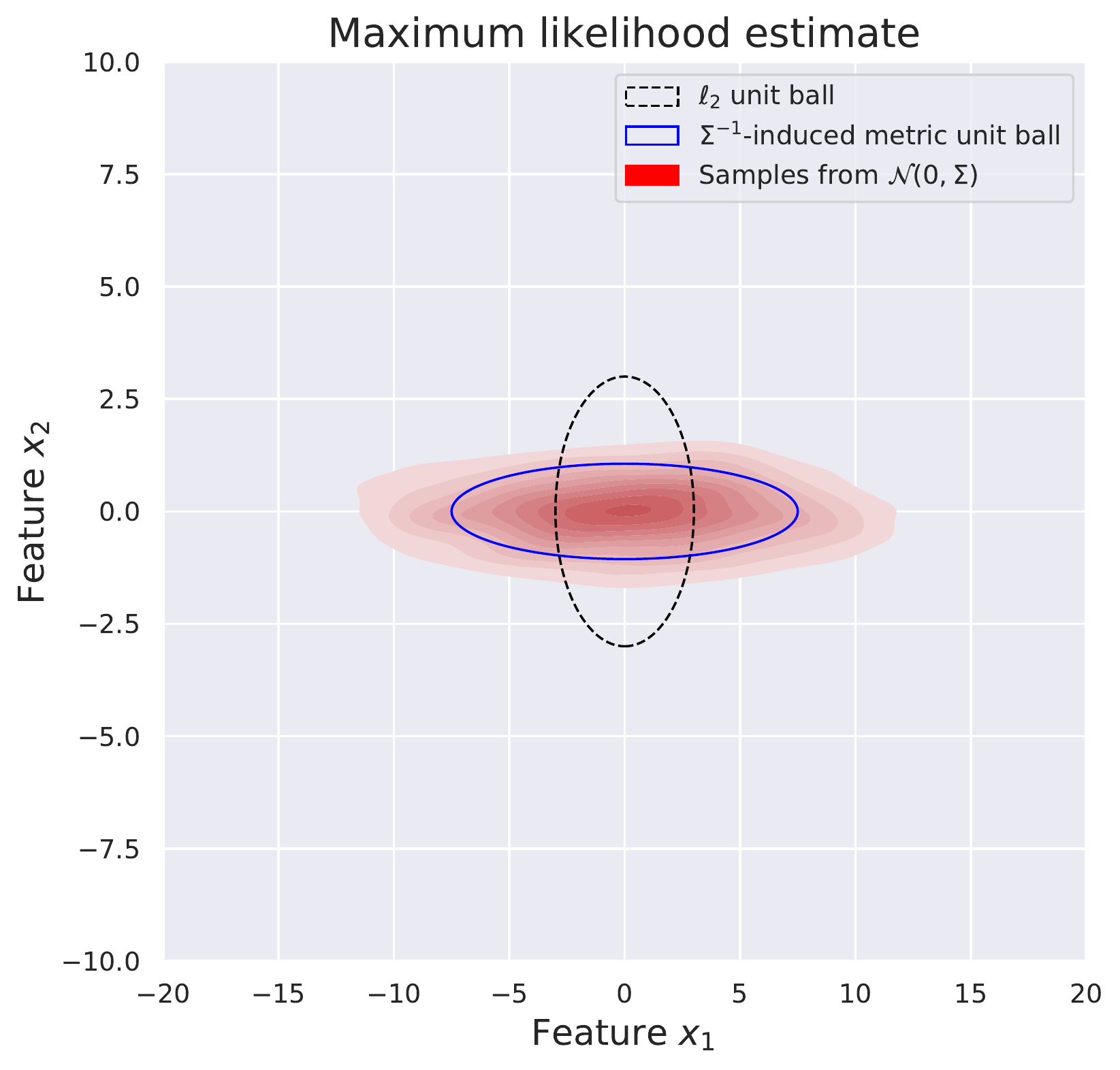}
	\includegraphics[width=0.24\textwidth]{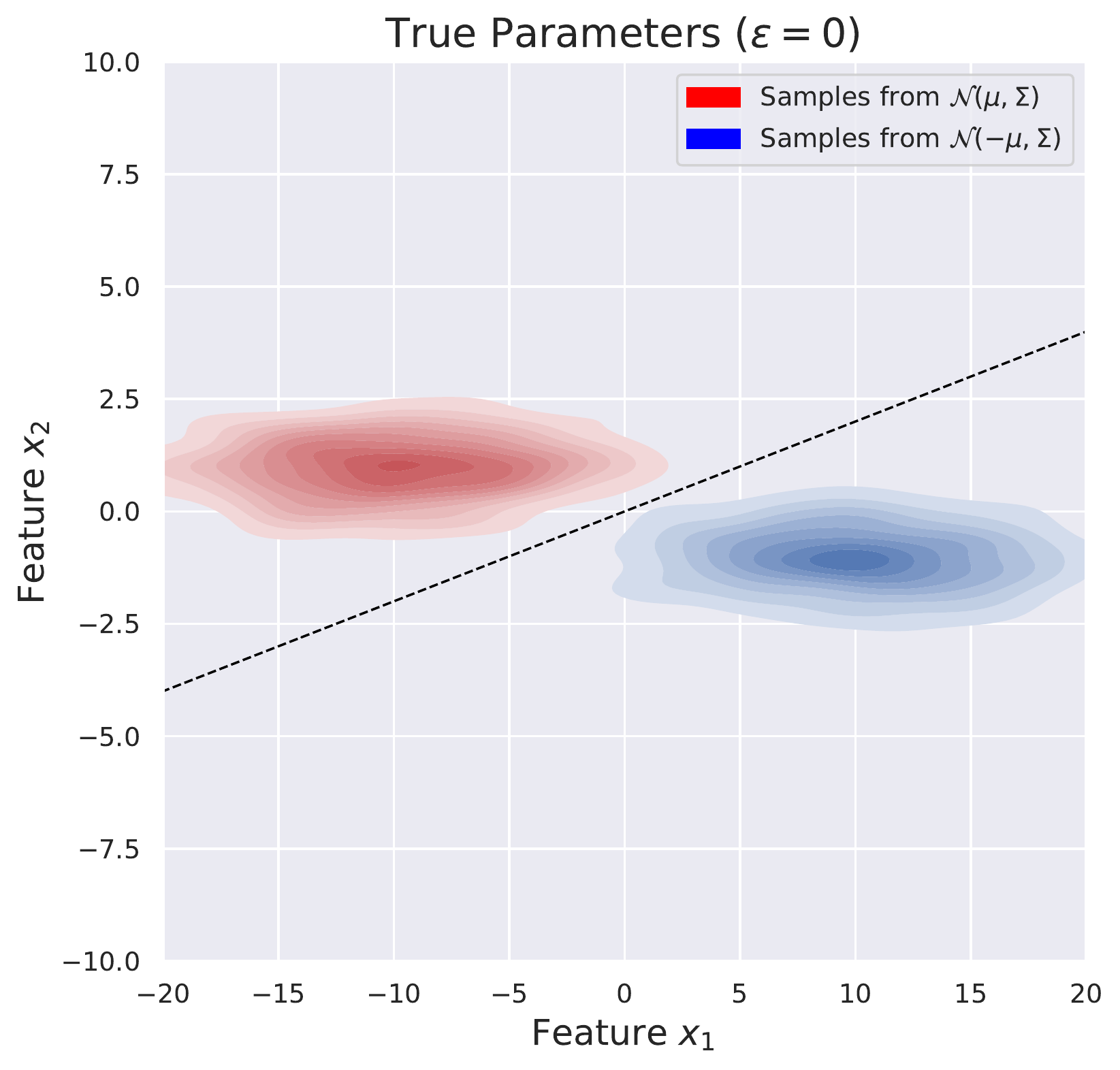}
	\includegraphics[width=0.24\textwidth]{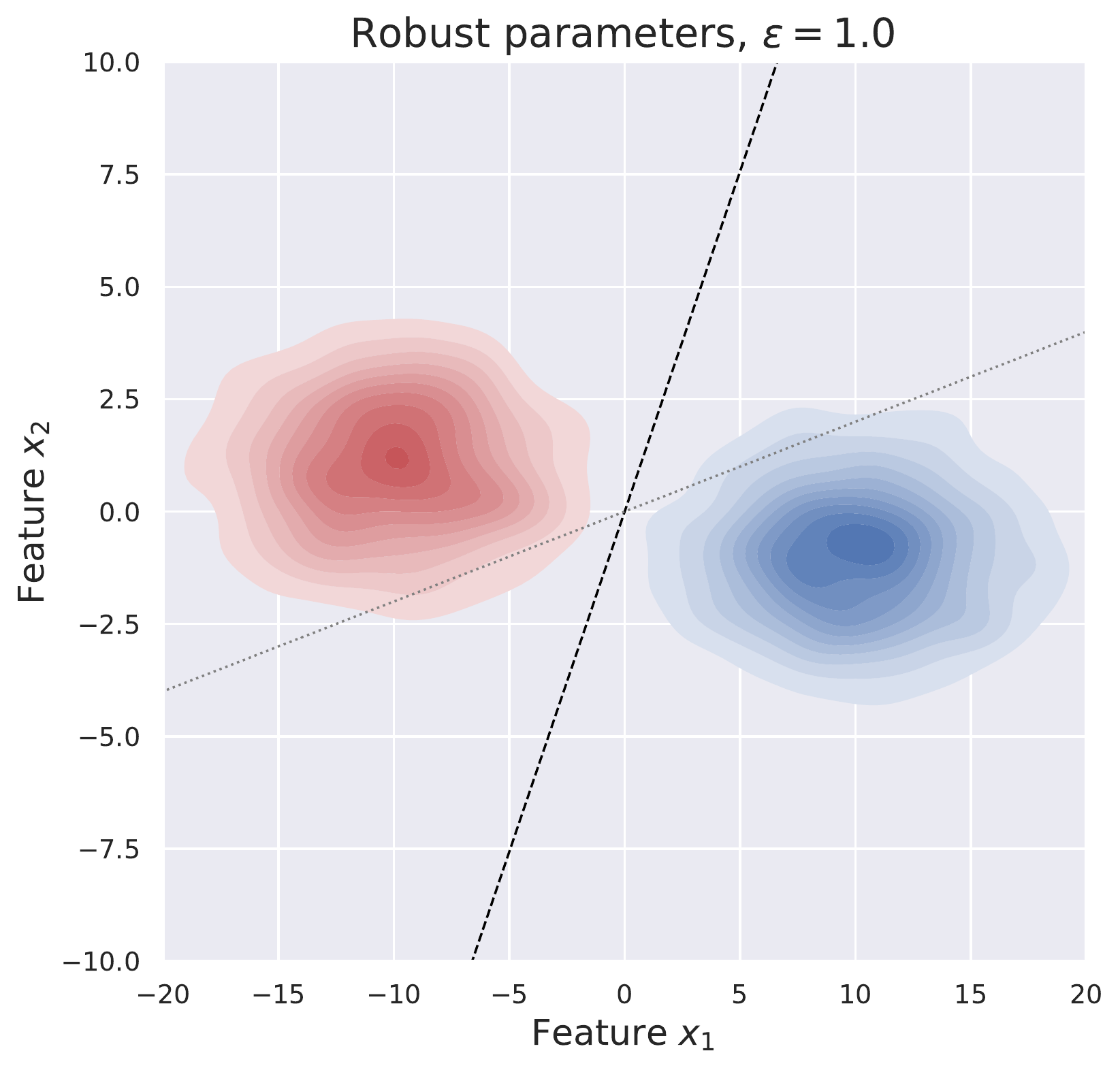}
	\includegraphics[width=0.24\textwidth]{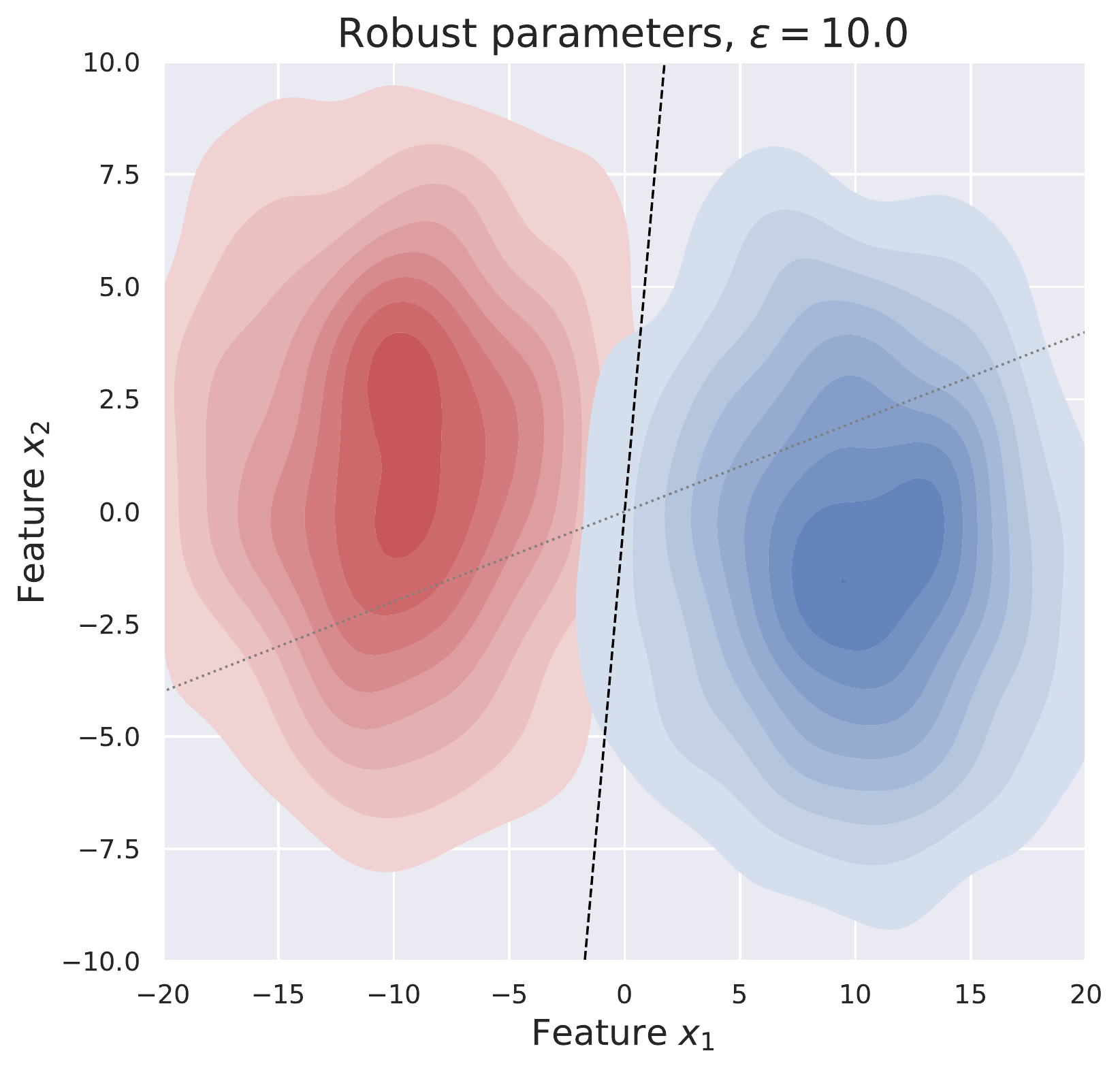}
	\caption{An empirical demonstration of the effect illustrated by
	    Theorem~\ref{thm:1}---as the adversarial perturbation budget
	    $\eps$ is increased, the learned mean $\mu$ remains constant, but the
	    learned covariance ``blends'' with the identity matrix,
	    effectively adding more and more uncertainty onto the
	    non-robust feature.}
	\label{fig:robust_effect}
    \end{center}
\end{figure}

\paragraph{(3) Gradient Interpretability.}~\citet{tsipras2019robustness} observe
that gradients of robust models tend to look more semantically meaningful.
It turns out that under our
model, this behaviour arises as a natural consequence of
Theorem~\ref{thm:1}. In particular, we show that the
resulting robustly learned parameters cause
the gradient of the linear classifier and the vector
connecting the means of the two distributions to better align (in a
worst-case sense) under the $\ell_2$ inner product. 
\begin{restatable}[Gradient alignment]{theorem}{gradients}
    \label{thm:2}
    Let $f(x)$ and $f_r(x)$ be monotonic classifiers based
    on the linear separator induced by standard and $\ell_2$-robust
    maximum likelihood classification, respectively. 
    The maximum angle formed between the gradient of the classifier (wrt
    input) and the vector connecting the classes can be smaller for the robust
    model:
    $$\min_{\mu} \frac{\inner{\bm{\mu}}{\nabla_x f_r(x)}}{
    \|\mu\|\cdot \|\nabla_x f_r(x)\|} > \min_{\mu}
\frac{\inner{\mu}{\nabla_x f(x)}}{\|\mu\|\cdot \|\nabla_x f(x)\|}.$$ 
\end{restatable}
Figure~\ref{fig:robust_effect} illustrates this phenomenon in the two-dimensional
case. With $\ell_2$-bounded adversarial training the gradient direction
(perpendicular to the decision boundary) becomes increasingly aligned under
the $\ell_2$ inner product with the vector between the means ($\mu$).

\paragraph{Discussion.} Our theoretical analysis suggests that rather than
offering any quantitative classification benefits, a
natural way to view the role of robust optimization is as enforcing a {\em
prior} over the features learned by the classifier. In particular, training with
an $\ell_2$-bounded adversary prevents the classifier from relying heavily on
features which induce a metric dissimilar to the $\ell_2$ metric. The
strength of the adversary then allows for a trade-off between the 
enforced prior, and the data-dependent features.

\paragraph{Robustness and accuracy.} Note that in the setting described so far,
robustness {\em can} be at odds with
accuracy since robust training prevents us from learning the most accurate
classifier (a similar
conclusion is drawn in~\citep{tsipras2019robustness}).
However, we note that there are very similar
settings where non-robust features manifest themselves in the same way, yet
a classifier with perfect robustness and accuracy is still attainable. 
Concretely, consider the distributions
pictured in Figure~\ref{fig:robustness_v_acc} in Appendix
~\ref{subsec:robustnessacc}. It is straightforward to show that
while there are many perfectly accurate classifiers, any standard loss
function will learn an accurate yet non-robust classifier. Only when robust
training is employed does the classifier learn a perfectly accurate and
perfectly robust decision boundary.

\section{Related Work}
\label{sec:related}
Several models for explaining adversarial examples have been proposed in
prior work, utilizing ideas ranging from finite-sample overfitting to
high-dimensional statistical
phenomena
~\citep{gilmer2018adversarial,fawzi2018adversarial,ford2019adversarial,tanay2016boundary,
shafahi2019are,mahloujifar2018curse,
shamir2019simple,goodfellow2015explaining,bubeck2018adversarial}.
The key differentiating aspect of our model is that adversarial
perturbations
arise as {\em well-generalizing, yet brittle, features}, rather than
statistical anomalies or effects of poor statistical concentration.
In particular, adversarial vulnerability does not stem from using a specific
model class or a specific training method, since standard training on the
``robustified'' data distribution of Section~\ref{sec:distill} leads to robust
models.
At the same time, as shown in Section~\ref{sec:distill_nrf}, these non-robust
features are sufficient to learn a good standard classifier.
We discuss the
connection between our model and others in detail in
Appendix~\ref{app:connections}.
We discuss additional related work in Appendix~\ref{sec:related_app}.

\section{Conclusion}
In this work, we cast the phenomenon of adversarial
examples as a natural consequence of the presence of {\em highly predictive
but non-robust features} in standard ML datasets. We
provide support for this hypothesis by explicitly disentangling robust and 
non-robust features in standard datasets, as well as showing
that non-robust features alone are sufficient for good generalization.
Finally, we study these phenomena in more detail in a theoretical
setting where we can rigorously study adversarial vulnerability, robust
training, and gradient alignment.

Our findings prompt us to view adversarial examples as a fundamentally {\em
human} phenomenon. In particular, we should not be
surprised that classifiers exploit highly predictive
features that happen to be non-robust under a human-selected notion of
similarity, given such features exist in real-world datasets. In the same manner, from
the perspective of interpretability,
as long as models rely on these non-robust features, we cannot expect to have model
explanations that are both human-meaningful and faithful to the models themselves.
Overall, attaining models that are robust and interpretable
will require explicitly encoding {\em human priors} into the training process.

\section{Acknowledgements}
We thank Preetum Nakkiran for suggesting the experiment of
Appendix~\ref{app:targeted_transfer} (i.e. replicating
Figure~\ref{fig:transfer} but with targeted attacks). We also are grateful to the authors of \citet{engstrom2019discussion} (Chris Olah, Dan
Hendrycks, Justin Gilmer, Reiichiro Nakano, Preetum Nakkiran, Gabriel Goh, Eric
Wallace)---for their insights and efforts replicating, extending, and discussing
our experimental results. 

Work supported in part by the NSF grants CCF-1553428, CCF-1563880,
CNS-1413920, CNS-1815221, IIS-1447786, IIS-1607189, the Microsoft Corporation, the Intel Corporation,
the MIT-IBM Watson AI Lab research grant, and an Analog Devices Fellowship.

\printbibliography

\clearpage

\appendix
\section{Connections to and Disambiguation from Other Models}
\label{app:connections}
Here, we describe other models for adversarial examples and how they relate
to the model presented in this paper.

\paragraph{Concentration of measure in high-dimensions.}
An orthogonal line of work~\citep{gilmer2018adversarial,fawzi2018adversarial,
mahloujifar2018curse,shafahi2019are}, argues that the high dimensionality
of the input space can present fundamental barriers on classifier robustness.
At a high level, one can show that, for certain data distributions, any decision
boundary will be close to a large fraction of inputs and hence no classifier can
be robust against small perturbations.
While there might exist such fundamental barriers to robustly classifying
standard datasets, this model cannot fully explain the situation observed in
practice, where one can train (reasonably) robust classifiers on
standard datasets~\citep{madry2018towards,raghunathan2018certified,
wong2018provable,xiao2019training,cohen2019certified}.

\paragraph{Insufficient data.}
\citet{schmidt2018adversarially} propose a theoretical model under which
a single sample is sufficient to learn a good, yet non-robust classifier,
whereas learning a good robust classifier requires $O(\sqrt{d})$ samples. 
Under this model, adversarial examples arise due to insufficient information
about the true data distribution.
However, unless the adversary is strong enough (in which case no robust
classifier exists), adversarial inputs cannot be utilized as inputs of the
opposite class (as done in our experiments in Section~\ref{sec:distill_nrf}).
We note that our model does not explicitly contradict the main thesis
of~\citet{schmidt2018adversarially}.
In fact, this thesis can be viewed as a natural
consequence of our conceptual framework.
In particular, since training models robustly reduces
the effective amount of information in the training data (as non-robust
features are discarded), more samples should be required to generalize
robustly.

\paragraph{Boundary Tilting.}
\citet{tanay2016boundary} introduce the
``boundary tilting'' model for adversarial examples, and suggest that
adversarial examples are a product of over-fitting. In particular, the model
conjectures that ``adversarial examples are possible because the class boundary
extends beyond the submanifold of sample data and can be---under certain
circumstances---lying close to it.'' Consequently, the authors suggest that
mitigating adversarial examples may be a matter of regularization and preventing
finite-sample overfitting.
In contrast, our empirical results in Section~\ref{sec:distill_nrf} suggest that
adversarial inputs consist of features inherent to the data distribution, since
they can encode generalizing information about the target class.

Inspired by this hypothesis and concurrently to our work,
\citet{kim2019bridging} present a simple classification task comprised of two
Gaussian distributions in two dimensions.
They experimentally show that the decision boundary tends to better align with
the vector between the two means for robust models.
This is a special case of our theoretical results in Section~\ref{sec:theory}.
(Note that this exact statement is not true beyond two dimensions, as discussed
in Section~\ref{sec:theory}.)

\paragraph{Test Error in Noise.}
\citet{fawzi2016robustness} and \citet{ford2019adversarial} argue that the
adversarial robustness of a classifier can be directly connected to its
robustness under (appropriately scaled) random noise.
While this constitutes a natural explanation
of adversarial vulnerability given the classifier robustness to noise, these
works do not attempt to justify the source of the latter. 

At the same time, recent
work~\cite{lecuyer2018certified,cohen2019certified,ford2019adversarial} utilizes
random noise during training or testing to construct adversarially robust
classifiers.
In the context of our framework, we can expect the added noise to
disproportionately affect non-robust features and thus hinder the model's
reliance on them.

\paragraph{Local Linearity.} \citet{goodfellow2015explaining} suggest that the
local linearity of DNNs is largely responsible for the existence of small
adversarial perturbations.
While this conjecture is supported by the effectiveness of adversarial attacks
exploiting local linearity (e.g., FGSM~\citep{goodfellow2015explaining}), it is
not sufficient to fully characterize the phenomena observed in practice.
In particular, there exist adversarial examples that violate the local linearity
of the classifier~\citep{madry2018towards}, while classifiers that are less
linear do not exhibit greater robustness~\citep{athalye2018obfuscated}.

\paragraph{Piecewise-linear decision boundaries.} \citet{shamir2019simple}
prove that the geometric structure of the classifier's decision boundaries can
lead to sparse adversarial perturbations.
However, this result does not take into account the distance to the decision
boundary along these direction or feasibility constraints on the input domain.
As a result, it cannot meaningfully distinguish between classifiers that are
brittle to small adversarial perturbations and classifiers that are moderately
robust.

\paragraph{Theoretical constructions which incidentally exploit non-robust features.}
\citet{bubeck2018adversarial} and \citet{nakkiran2019adversarial} propose
theoretical models where the barrier to learning robust classifiers is,
respectively, due to computational constraints or model complexity.
In order to construct distributions that admit accurate yet non-robust
classifiers they (implicitly) utilize the concept of non-robust features.
Namely, they add a low-magnitude signal to each input that encodes the true
label.
This allows a classifier to achieve perfect standard accuracy, but cannot be
utilized in an adversarial setting as this signal is susceptible to
small adversarial perturbations.

\section{Additional Related Work}
\label{sec:related_app}
We describe previously proposed models for the existence of adversarial examples
in the previous section.
Here we discuss other work that is methodologically or conceptually similar to
ours.

\paragraph{Distillation.} The experiments performed in
Section~\ref{sec:distill} can be seen as a form of {\em distillation}.
There is a line of work, known as model
distillation~\citep{hinton2015distilling,furlanello2018born, bucilua2006model},
where the goal is to train a new model to mimic another already trained model.
This is typically achieved by adding some regularization terms to the loss in
order to encourage the two models to be similar, often replacing training labels
with some other target based on the already trained model.
While it might be possible to successfully distill a robust model using these
methods, our goal 
was to achieve it by {\em only} modifying the training set (leaving the
training process unchanged), hence demonstrating that adversarial vulnerability
is mainly a property of the dataset.
Closer to our work is dataset
distillation~\citep{wang2018dataset} which considers the problem of
reconstructing a classifier from an alternate dataset much smaller than the
original training set.
This method aims to produce inputs that directly encode the weights of the
already trained model by ensuring that the classifier's gradient with respect to
these inputs approximates the desired weights. (As a result, the inputs
constructed do not resemble natural inputs.)
This approach is orthogonal to our goal since we are not interested in encoding
the particular weights into the dataset but rather in imposing a structure
to its features.

\paragraph{Adversarial Transferabiliy.} In our work, we posit that a
potentially natural consequence of the existence of non-robust features is
{\em adversarial
transferability}~\citep{papernot2017practical,liu2017delving,papernot2016transferability}.
A recent line of work has considered this phenomenon from a theoretical
perspective, confined to simple models, or unbounded
perturbations~\cite{charles2019geometric,zou2017geometric}.
~\citet{tramer2017space} study transferability empirically, by finding {\em
adversarial subspaces}, (orthogonal vectors whose linear combinations
are adversarial perturbations). The authors find that there
is a significant overlap in the adversarial subspaces between different
models, and identify this as a source of transferability. In our work, we
provide a potential reason for this overlap---these directions correspond to
non-robust features utilized by models in a similar manner.

\paragraph{Universal Adversarial Perturbations} \citet{moosavi2017universal}
construct perturbations that can cause misclassification when applied to
multiple different inputs.
More recently, \citet{jetley2018friends} discover input patterns that are
meaningless to humans and can induce misclassification, while at the same time
being essential for standard classification.
These findings can be naturally cast into our framework by considering these
patterns as non-robust features, providing further evidence about their
pervasiveness.

\paragraph{Manipulating dataset features}
\citet{ding2018on} perform synthetic transformations on the dataset (e.g., image
saturation) and study the performance of models on the transformed dataset under
standard and robust training. While this can be seen as a method of restricting
the features available to the model during training, it is unclear how well
these models would perform on the standard test set.
\citet{geirhos2018imagenettrained} aim to quantify the relative dependence of
standard models on shape and texture information of the input.
They introduce a version of ImageNet where texture information has been removed
and observe an improvement to certain corruptions.

\section{Experimental Setup}
\label{app:experimental_setup}

\subsection{Datasets}
\label{sec:dss}
For our experimental analysis, we use the CIFAR-10~\cite{krizhevsky2009learning}
and (restricted) ImageNet ~\cite{russakovsky2015imagenet} datasets. Attaining
robust models for the complete ImageNet dataset is known to be a challenging
problem, both due to the hardness of the learning problem itself, as well as the
computational complexity.  We thus restrict our focus to a subset of the dataset
which we denote as restricted ImageNet. To this end, we group together
semantically similar classes from ImageNet into 9 super-classes shown in
Table~\ref{tab:classes}.  We train and evaluate only on examples corresponding
to these classes.

\begin{table}[!h]
\caption{Classes used in the Restricted ImageNet model. The class ranges are
inclusive.}
\begin{center}
  \begin{tabular}{ccc}
    \toprule
    \textbf{Class} & \phantom{x} & \textbf{Corresponding ImageNet Classes} \\
    \midrule
    ``Dog'' &&   151  to 268    \\ 
    ``Cat'' &&   281  to 285    \\
    ``Frog'' &&   30  to 32    \\
    ``Turtle'' &&   33  to 37    \\
    ``Bird'' &&   80  to 100    \\
    ``Primate'' &&   365  to 382    \\
    ``Fish'' &&   389  to 397    \\
    ``Crab'' &&   118  to 121    \\
    ``Insect'' &&   300  to 319    \\
    \bottomrule
  \end{tabular}
\end{center}
\label{tab:classes}
\end{table}

\subsection{Models}
\label{sec:models}
We use the ResNet-50 architecture for our baseline standard and adversarially
trained classifiers on CIFAR-10 and restricted ImageNet. For each model, we
grid search over three learning rates ($0.1$, $0.01$, $0.05$), two
batch sizes ($128$, $256$) including/not including a learning rate drop
(a single order of magnitude) and data augmentation. We use the standard
training parameters for the remaining parameters. The hyperparameters used
for each model are given in Table~\ref{tab:all_hyperparams}.

\begin{table}[!h]
\caption{Hyperparameters for the models trained in the main paper. All
hyperparameters were obtained through a grid search.}
\begin{center}
  \begin{tabular}{l|cccccc}
    \toprule
    {\bf Dataset} & LR & Batch Size & LR Drop & Data Aug. & Momentum & Weight Decay \\
    \midrule
    $\drobust$ (CIFAR) & 0.1 & 128 & Yes & Yes & 0.9 & $5\cdot 10^{-4}$ \\
    $\drobust$ (Restricted ImageNet) &  0.01 & 128 & No & Yes & 0.9 & $5\cdot 10^{-4}$ \\
    $\dnonrobust$ (CIFAR) & 0.1 & 128 & Yes & Yes & 0.9 & $5\cdot 10^{-4}$ \\
    $\drand$ (CIFAR) & 0.01 & 128 & Yes & Yes & 0.9 & $5\cdot 10^{-4}$ \\
    $\drand$ (Restricted ImageNet) & 0.01 & 256 & No & No & 0.9 & $5\cdot 10^{-4}$ \\
    $\ddet$ (CIFAR) & 0.1 & 128 & Yes & No & 0.9 & $5\cdot 10^{-4}$ \\
    $\ddet$ (Restricted ImageNet) & 0.05 & 256 & No & No & 0.9 & $5\cdot 10^{-4}$  \\
    \bottomrule
  \end{tabular}
\end{center}
\label{tab:all_hyperparams}
\end{table}

\clearpage
\subsection{Adversarial training}
\label{sec:robo}
To obtain robust classifiers, we employ the adversarial training methodology
proposed in~\cite{madry2018towards}.  Specifically, we train against a projected
gradient descent (PGD) adversary constrained in $\ell_2$-norm starting from the
original image.
Following \citet{madry2018towards} we normalize the gradient at each step of PGD
to ensure that we move a fixed distance in $\ell_2$-norm per step.
Unless otherwise specified, we use the values of $\epsilon$ provided in
Table~\ref{tab:eps} to train/evaluate our models.
We used $7$ steps of PGD with a step size of $\eps/5$.

\begin{table}[!h]
\caption{Value of $\epsilon$ used for  $\ell_2$ adversarial training/evaluation
of each dataset.}
\begin{center}
  \begin{tabular}{cccccc}
    \toprule
    \textbf{Adversary} & \phantom{x} &  \textbf{CIFAR-10} & \textbf{Restricted Imagenet} \\
    \midrule
    $\ell_2$ && 0.5 & 3 \\
    \bottomrule
  \end{tabular}
\end{center}
\label{tab:eps}
\end{table}

\subsection{Constructing a Robust Dataset}

In Section~\ref{sec:distill}, we describe a procedure to construct a dataset that
contains features relevant only to a given (standard/robust) model. To do so, we
optimize the training objective in \eqref{eq:robustify}. Unless otherwise
specified, we initialize $x_r$ as a different randomly chosen sample from the
training set. (For the sake of completeness, we also try initializing with a
Gaussian noise instead as shown in Table~\ref{tab:robustify_cifar_noise}.) We
then perform normalized gradient descent ($\ell_2$-norm of gradient is fixed to
be constant at each step). At each step we clip the input $x_r$ to in the $[0,1]$
range so as to ensure that it is a valid image. Details on the optimization
procedure are shown in Table~\ref{tab:distill_params}. 
We provide the pseudocode for the construction in
Figure~\ref{alg:robustification}.

\begin{figure}[h!]
	\frame{
		\begin{minipage}{0.9\textwidth}
			\vspace{10pt}
			\hspace{5pt}
			$\textsc{GetRobustDataset}(D)$
			
			\begin{enumerate}
				\item
				$C_R \leftarrow $  \textsc{AdversarialTraining}($D$) \\
				$g_{R} \leftarrow$ mapping learned by $C_R$ from the input to the representation layer

				\item $D_R \leftarrow \{\}$ 
				\item For $(x, y) \in D$ 
				\subitem $x' \sim D$
                \subitem$ x_R \leftarrow \arg\min_{z\in[0,1]^d} \|g_R(z) - g_R(x)\|_2$  
				\hfil \eqparbox{COMMENT}{\# Solved using $\ell_2$-PGD starting
                from $x'$}
				\subitem $D_R \leftarrow D_R \; \bigcup \;  \{(x_R, y) \}$
				\item Return $D_R$
			\end{enumerate}
			\vspace{5pt}
		\end{minipage}
	}
	\caption{Algorithm to construct a ``robust'' dataset, by restricting
		to features used by a robust model. }
	\label{alg:robustification}
\end{figure}

\begin{table}[!h]
\caption{Parameters used for optimization procedure to construct dataset in
Section~\ref{sec:distill}. }
\begin{center}
  \begin{tabular}{cccccc}
    \toprule
    \textbf{}  & \textbf{CIFAR-10} & \textbf{Restricted Imagenet} \\
    \midrule
    step size  &  0.1 & 1 \\ 
    iterations  & 1000  & 2000 \\
    \bottomrule
  \end{tabular}
\end{center}
\label{tab:distill_params}
\end{table}

\clearpage
\subsection{Non-robust features suffice for standard classification}
To construct the dataset as described in Section~\ref{sec:distill_nrf}, we use
the standard projected gradient descent (PGD) procedure described
in~\cite{madry2018towards} to construct an adversarial example for a given input
from the dataset~\eqref{eqn:adv}.
Perturbations are constrained in $\ell_2$-norm while each PGD step is normalized
to a fixed step size.
The details for our PGD setup are described in
Table~\ref{tab:distill_params_nrf}.
We provide pseudocode in Figure~\ref{alg:nrf}.

\begin{figure}[h!]
	\frame{
		\begin{minipage}{.9\textwidth}
			\vspace{10pt}
			\hspace{5pt}
			$\textsc{GetNonRobustDataset}(D, \eps)$
			
			\begin{enumerate}
				\item $D_{NR} \leftarrow \{\}$
				\item
				$C \leftarrow $  \textsc{StandardTraining}($D$) 				
				\item For $(x, y) \in D$ 
                    \subitem $t \stackrel{\text{uar}}{\sim} [C]$
                    \hfil \eqparbox{COMMENT}{\# or $t
                                    \leftarrow (y + 1) \mod C $}
                \subitem$ x_{NR} \leftarrow \min_{||x' - x|| \leq \eps}
                        L_C(x', t)$  
				\hfil \eqparbox{COMMENT}{\# Solved using $\ell_2$ PGD}
                \subitem $D_{NR} \leftarrow D_{NR} \; \bigcup \;  \{(x_{NR}, t) \}$
            \item Return $D_{NR}$
			\end{enumerate}
			\vspace{5pt}
		\end{minipage}
	}
	\caption{Algorithm to construct a dataset where input-label
	association is based entirely on non-robust features. }
	\label{alg:nrf}
\end{figure}

\begin{table}[!h]
\caption{Projected gradient descent parameters used to construct constrained
adversarial examples in Section~\ref{sec:distill_nrf}.}
\begin{center}
  \begin{tabular}{cccccc}
    \toprule
    \textbf{Attack Parameters}  & \textbf{CIFAR-10} & \textbf{Restricted Imagenet} \\
    \midrule
    $\eps$ & 0.5 & 3 \\
    step size  &  0.1 & 0.1 \\ 
    iterations  & 100  & 100 \\
    \bottomrule
  \end{tabular}
\end{center}
\label{tab:distill_params_nrf}
\end{table}

\clearpage
\section{Omitted Experiments and Figures}
\label{app:omitted_figures}

\subsection{Detailed evaluation of models trained on ``robust'' dataset}
\label{app:from_noise}

In Section~\ref{sec:distill}, we generate a ``robust'' training set by
restricting the dataset to only contain features relevant to a robust model
(robust dataset) or a standard model (non-robust dataset).
This is performed by choosing either a random input from the training set
or random noise\footnote{We use 10k steps to construct the dataset from noise, 
instead to using 1k steps done when the input is a different training set 
image (cf. Table~\ref{tab:distill_params}).} and then
performing the optimization procedure described in~\eqref{eq:robustify}.
The performance of these classifiers along with various baselines
is shown in Table~\ref{tab:robustify_cifar}. 
We observe that while the robust dataset constructed from noise resembles the
original, the corresponding non-robust does not
(Figure~\ref{fig:from_noise_inputs}).
This also leads to suboptimal performance of classifiers trained on this dataset
(only $46\%$ standard accuracy) potentially due to a distributional shift.

\begin{table}[!htp]
    \caption{Standard and robust classification performance on the CIFAR-10 test
        set of: an (i) ERM classifier; (ii) ERM classifier trained on a dataset
        obtained by distilling features relevant to ERM classifier in (i);
        (iii) adversarially trained classifier ($\eps=0.5$); (iv) ERM classifier
        trained on dataset obtained by distilling features used by robust
        classifier in (iii). Simply restricting the set of available features
        during ERM to features used by a standard model yields non-trivial
    robust accuracy.}
	\label{tab:robustify_cifar}
	\label{tab:robustify_cifar_noise}
    \setlength{\tabcolsep}{.3cm}
	\begin{center}
		\begin{tabular}{lrrr}
			\toprule
            & & \multicolumn{2}{c}{\bf Robust Accuracy}\\
            \textbf{Model} &  \textbf{Accuracy} &
            $\eps=0.25$ & $\eps=0.5$ \\
			\midrule
			Standard Training &  95.25 \% & 4.49\% &  0.0\% \\
			Robust Training &  90.83\% & 82.48\% & 70.90\% \\
			\midrule
            Trained on non-robust dataset (constructed from images)
                   &  \textbf{87.68\%} & 0.82\%  &  0.0\% \\
			Trained on non-robust dataset (constructed from noise)
                   &  \textbf{45.60\%} & 1.50\% &  0.0\% \\ 
			Trained on robust dataset (constructed from images)
                   &  {85.40\%}   & \textbf{48.20 \%} & 21.85\% \\
			Trained on robust dataset (constructed from noise)
                   &  {84.10\%}   & \textbf{48.27 \%} & 29.40\% \\
			\bottomrule
		\end{tabular}
	\end{center}
\end{table}

\begin{figure}[htp]
	\begin{center}
        \includegraphics[width=0.5\textwidth]{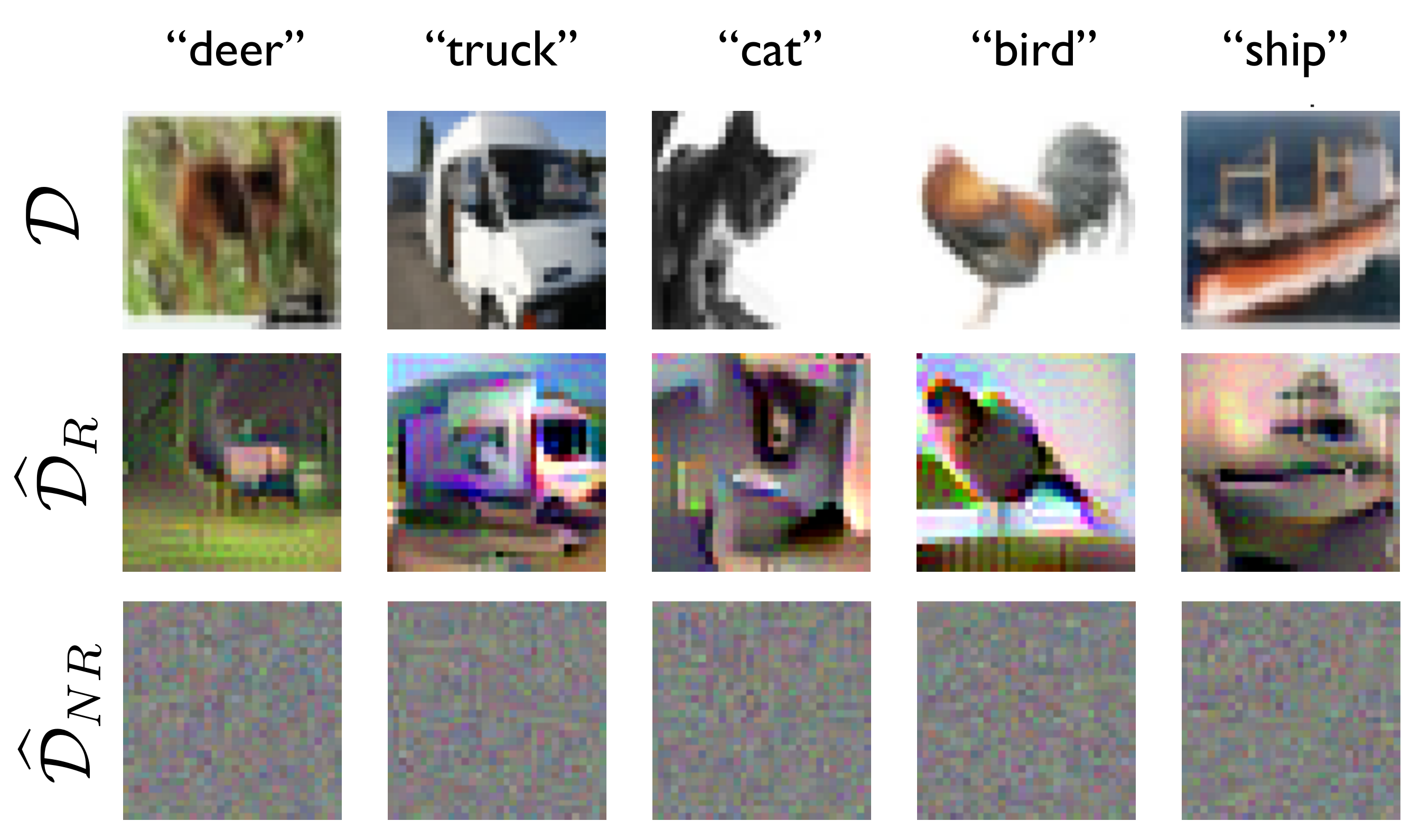}
        \caption{Robust and non-robust datasets for CIFAR-10 when the process
            starts from noise (as opposed to random images as in
            Figure~\ref{fig:robust_inputs}).}
        \label{fig:from_noise_inputs}
	\end{center}
\end{figure}

\clearpage
\subsection{Adversarial evaluation}
\label{app:attacks}
To verify the robustness of our classifiers trained on the `robust''
dataset, we evaluate them with strong attacks~\citep{carlini2019on}. In
particular, we try up to 2500 steps of projected gradient descent (PGD),
increasing steps until the accuracy plateaus, and also try the CW-$\ell_2$
loss function~\citep{carlini2017towards} with 1000 steps. For each attack
we search over step size. We find that over all attacks and step sizes, the
accuracy of the model does not drop by more than 2\%, and plateaus at
$48.27\%$ for both PGD and CW-$\ell_2$ (the value given in
Figure~\ref{fig:distill_results}). We show a plot of accuracy in terms of
the number of PGD steps used in Figure~\ref{fig:accvpgd}.

\begin{figure}[h!]
    \begin{center}
	\includegraphics[width=0.5\textwidth]{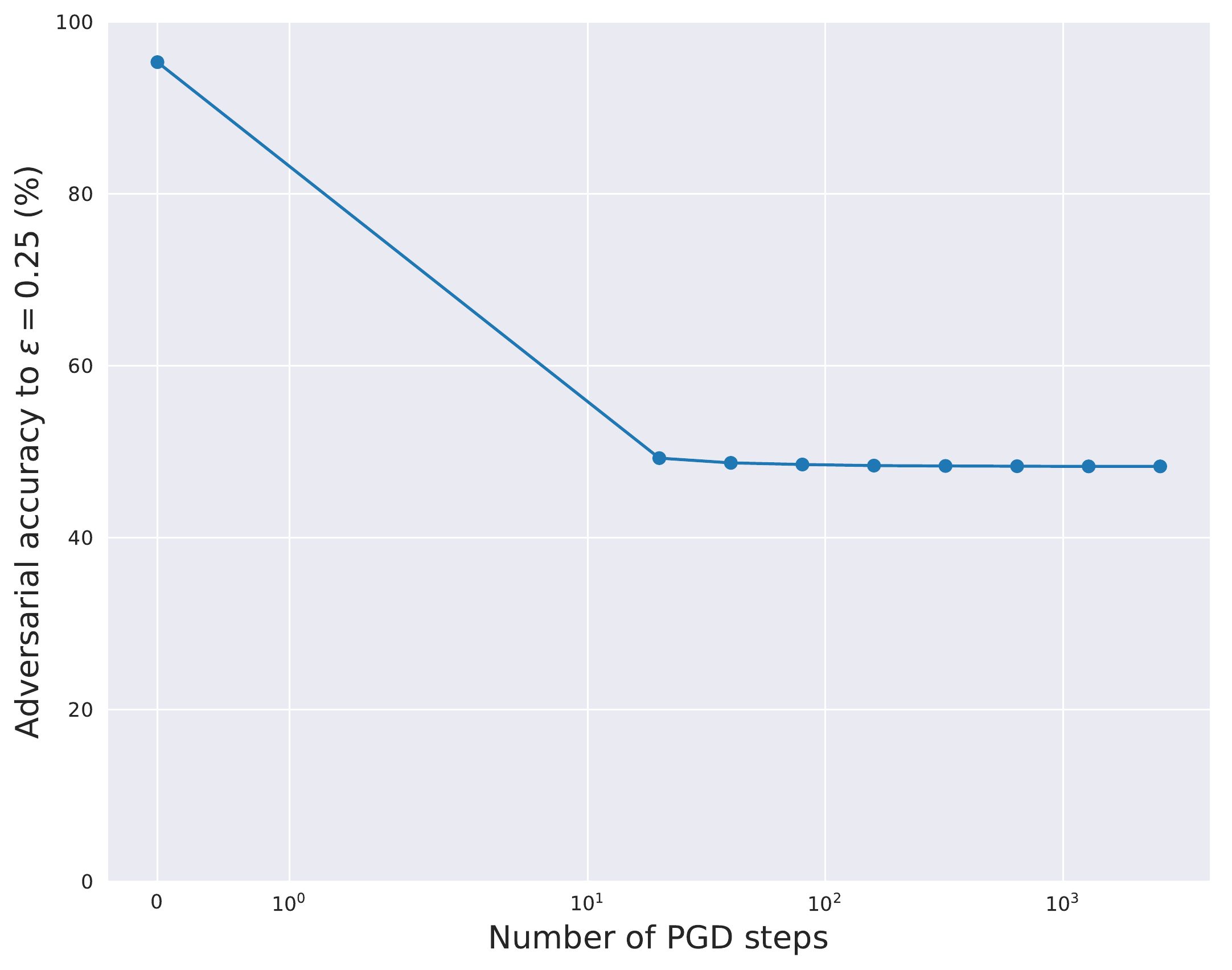}
	\caption{Robust accuracy as a function of the number of PGD steps
	used to generate the attack. The accuracy plateaus at $48.27\%$.}
	\label{fig:accvpgd}
    \end{center}
\end{figure}

\subsection{Performance of ``robust'' training and test set}
\label{app:robust_test}
In Section~\ref{sec:distill}, we observe that an ERM classifier trained on a
``robust'' training dataset $\drobust$ (obtained by restricting features
to those relevant to a robust model) attains non-trivial robustness (cf.
Figure~\ref{fig:distill} and Table~\ref{tab:robustify_cifar}).
In Table~\ref{tab:robust_test_set}, we evaluate the adversarial accuracy of
the model on the corresponding robust training set (the samples which the
classifier was trained on) and test set (unseen samples from $\drobust$,
based on the test set).
We find that the drop in robustness comes from a combination of generalization
gap (the robustness on the $\drobust$ test set is worse than it is on the
robust training set) and distributional shift (the model performs better on
the robust test set consisting of unseen samples from $\drobust$ than on
the standard test set containing unseen samples from $\D$).

\begin{table}[!htp]
    \caption{Performance of model trained on the \emph{robust dataset} on the
        robust training and test sets as well as the standard CIFAR-10 test set.
        We observe that the drop in robust accuracy stems from a combination of
        generalization gap and distributional shift. The adversary is
        constrained to $\eps=0.25$ in $\ell_2$-norm.
    }
    \label{tab:robust_test_set}
\begin{center}
    \setlength{\tabcolsep}{1cm}
    \begin{tabular}{lr}
    \toprule
    Dataset & Robust Accuracy \\
    \midrule
    Robust training set & 77.33\% \\
    Robust test set & 62.49\% \\
    Standard test set & 48.27\% \\
    \bottomrule
\end{tabular}
\end{center}
\end{table}

\subsection{Classification based on non-robust features}
\label{app:nrf}
Figure~\ref{fig:distill_nonrobust} shows sample images from $\mathcal{D}$,
$\drand$ and $\ddet$ constructed using a standard (non-robust) ERM
classifier, and an adversarially trained (robust) classifier.
\begin{figure}[h!]
	
	\begin{subfigure}[b]{0.45\textwidth}
		\begin{center}
			\includegraphics[width=1.0\textwidth]{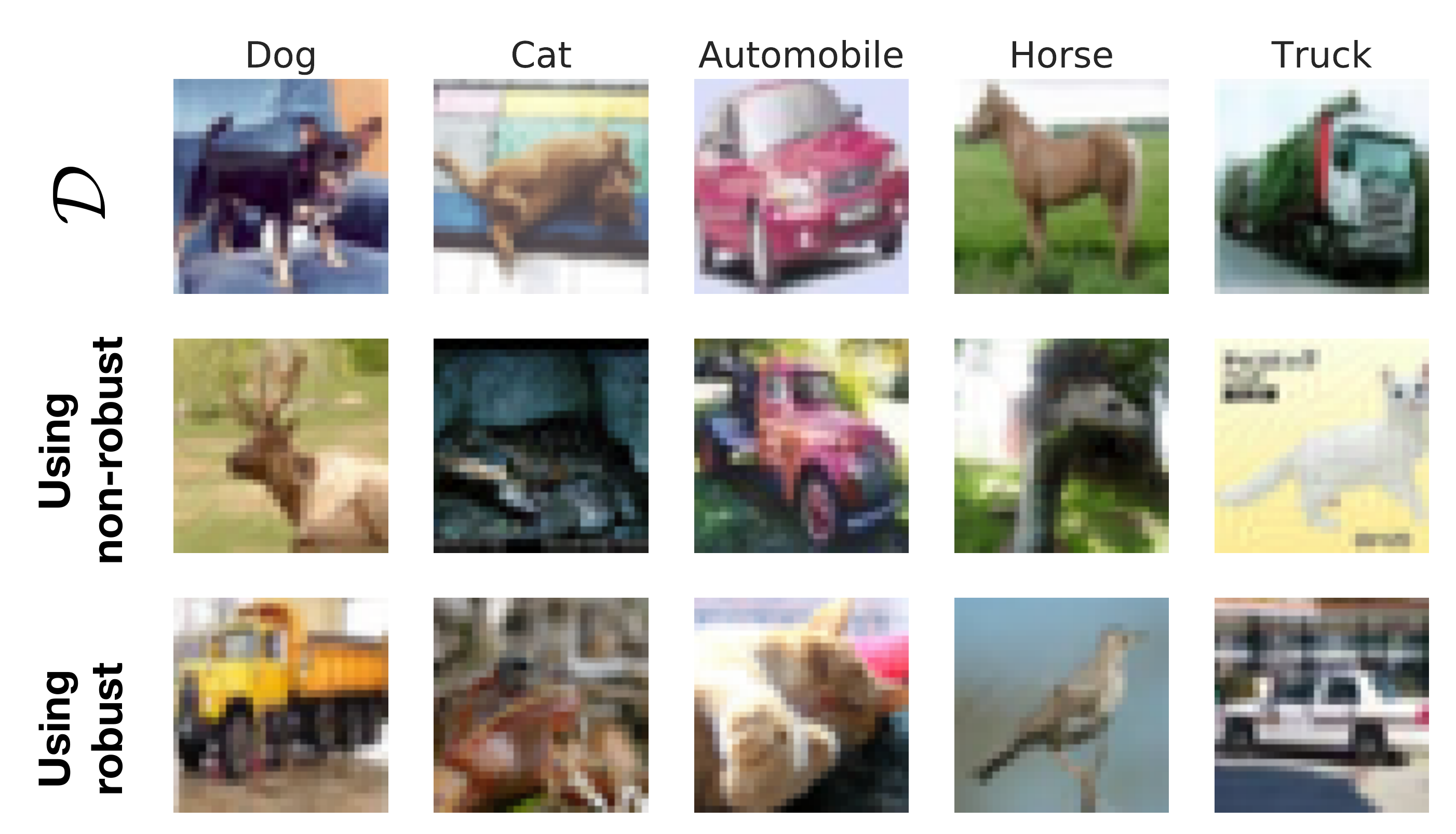}
			\label{fig:rand}
			\caption{$\drand$}
		\end{center}
	\end{subfigure}
	\hfil
	\begin{subfigure}[b]{0.45\textwidth}
		\begin{center}
			\includegraphics[width=1.0\textwidth]{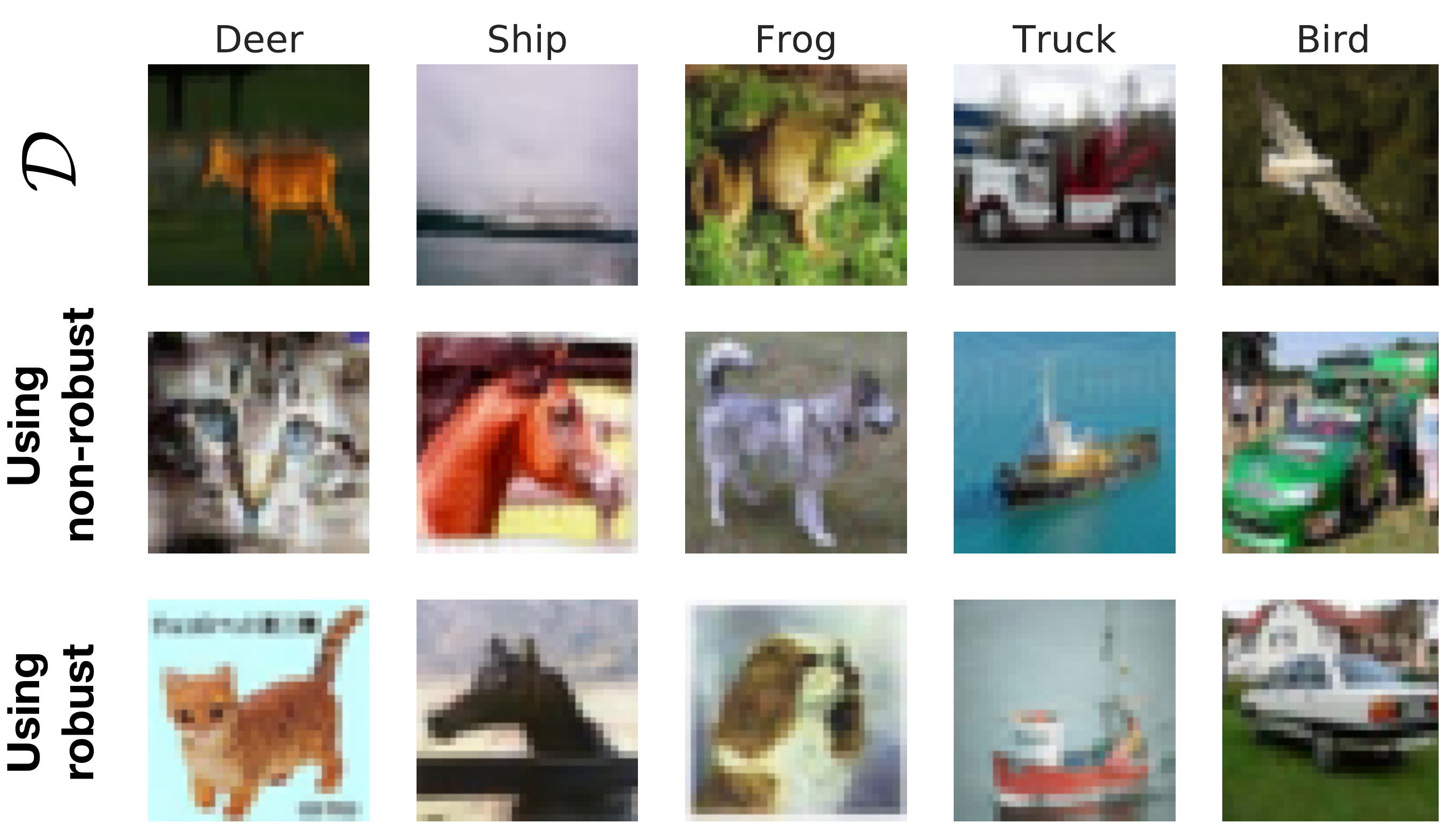}
			\label{fig:determ}
			\caption{$\ddet$}
		\end{center}
	\end{subfigure}
	\caption{Random samples from datasets where the input-label correlation is
		entirely based on non-robust features. Samples are generated by performing
		small adversarial perturbations using either random ($\drand$) or 
		deterministic ($\ddet$) label-target mappings
		for every sample in the training set. Each image shows: {\em top}: original;
		{\em middle}: adversarial perturbations using a standard ERM-trained classifier;
		{\em bottom}: adversarial perturbations using a robust classifier (adversarially
		trained against $\eps=0.5$).}
	\label{fig:distill_nonrobust}
\end{figure}

In Table~\ref{tab:distill_nrf_perf_app}, we repeat the experiments in 
Table~\ref{tab:adv_next} based on datasets constructed using
a robust model. Note that using a robust model to generate the $\ddet$ 
and $\drand$ datasets will not result in non-robust features that
are strongly predictive of $t$ (since the prediction of the classifier
will not change).
Thus, training a model on these datasets leads to poor accuracy on the
standard test set from $\mathcal{D}$.

Observe from Figure~\ref{fig:training_curves} that models trained on
datasets derived from the robust model show a decline in test accuracy
as training progresses. In Table~\ref{tab:distill_nrf_perf_app}, the accuracy
numbers reported correspond to the \emph{last} iteration, and 
not the \emph{best} performance. This is because we have
no way to cross-validate in a meaningful way as the validation 
set itself comes from $\drand$ or $\ddet$, and not from the true data 
distribution $D$. Thus, validation accuracy will not be predictive of
the true test accuracy, and thus will not help determine when to early 
stop.

\begin{table}[!h]
    \caption{Repeating the experiments of Table~\ref{tab:adv_next} using a
    robust model to construct the datasets $\D$, $\drand$ and $\ddet$. 
    Results in Table~\ref{tab:adv_next}
    are reiterated for comparison.}
	\begin{center}
		\begin{tabular}{cccccc}
			\toprule
			\multirow{2}{*}{\shortstack[c]{\bf Model used  \\ \bf to construct dataset}} &&
			\multicolumn{3}{c}{{\bf Dataset 
			used in training}} \\ 
			\cmidrule{3-5}
		     && $\mathcal{D}$ & $\drand$ & $\ddet$ \\
			\midrule
			 Robust &&  95.3\% & 25.2 \% & {5.8\%} \\ 
			Standard &&  95.3\% & 63.3 \% & {43.7\%} \\ 
			\bottomrule
		\end{tabular}
	\end{center}
	\label{tab:distill_nrf_perf_app}
\end{table}

\clearpage
\subsection{Accuracy curves}
\label{app:accuracy_curves}

\begin{figure}[h!]
	\begin{center}
		\begin{subfigure}[b]{1.0\textwidth}
			\includegraphics[width=1.0\textwidth]{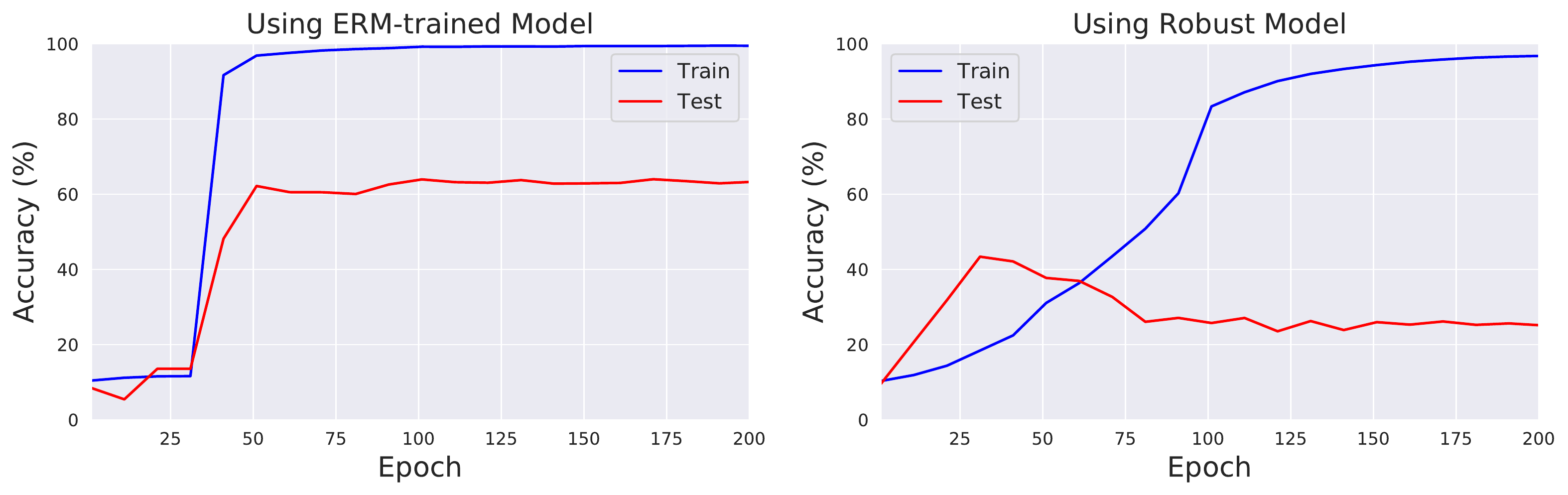}
			\caption{Trained using $\drand$ training set}
		\end{subfigure}
		\begin{subfigure}[b]{1.0\textwidth}
			\includegraphics[width=1.0\textwidth]{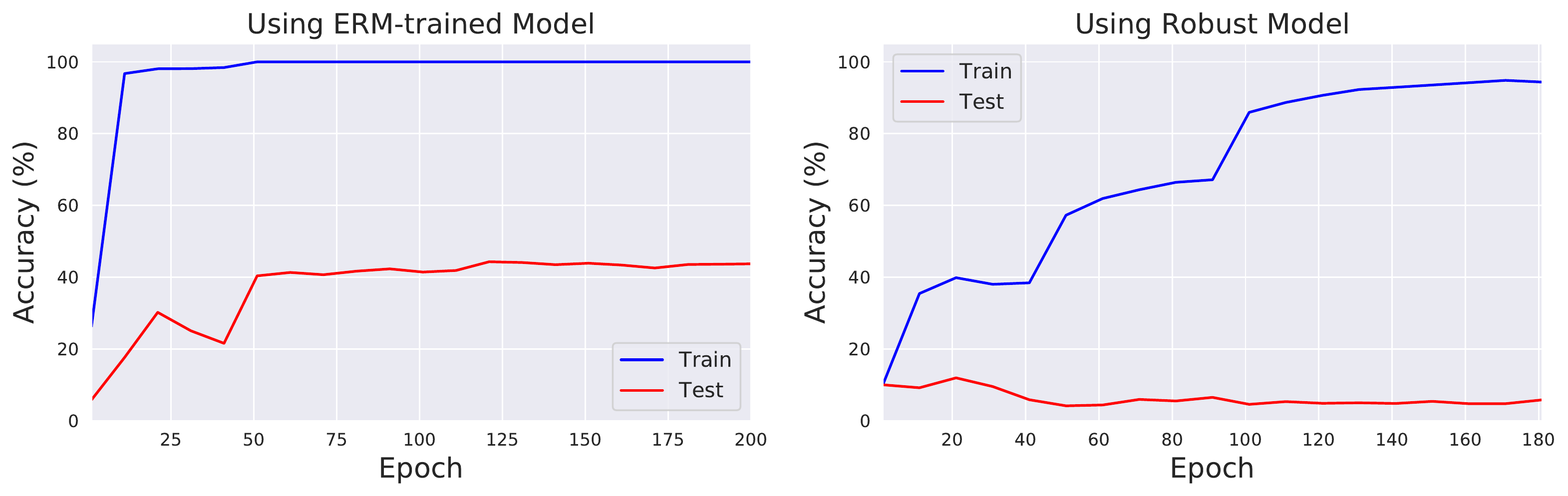}
			\caption{Trained using $\ddet$ training set}
		\end{subfigure}
	\end{center}
	\caption{Test accuracy on $\mathcal{D}$ of standard classifiers 
		trained on datasets where input-label correlation is based solely on non-robust
		features as in Section~\ref{sec:distill_nrf}.
		The datasets are constructed using either a non-robust/standard model 
		(\emph{left column}) or a robust model (\emph{right column}). 
		The labels used are either random ($\drand$;
		\emph{top row}) or correspond to a deterministic permutation 
		($\ddet$; \emph{bottom row}).
	}
	\label{fig:training_curves}
\end{figure}

\clearpage
\subsection{Performance of ERM classifiers on relabeled test set}
\label{app:erm_relabeled}
In Table~\ref{tab:adv_next_CIFAR_eval_next}), we evaluate the
performance of classifiers trained on $\ddet$ on both the original 
test set drawn from $\mathcal{D}$, and the test set relabelled using
$t(y) = (y + 1) \mod C$. Observe that the classifier trained on $\ddet$ 
constructed using a robust model actually ends up learning permuted labels 
based on robust features (indicated by high test accuracy on the relabelled
test set). 

\begin{table}[!h]
    \caption{
        Performance of classifiers trained using $\ddet$
        training set constructed using either standard or robust models. 
        The classifiers are evaluated both on the standard test set from $\mathcal{D}$
        and the test set relabeled using $t(y) = (y + 1) \mod C$.
        We observe that using a robust model for the construction results in a
        model that largely predicts the permutation of labels, indicating that
        the dataset does not have strongly predictive non-robust features.}
	\begin{center}
		\begin{tabular}{ccccccc}
			\toprule
			\multirow{2}{*}{\shortstack[c]{\bf Model used to construct \\ \bf training dataset for 
			$\ddet$}} &&& \multicolumn{4}{c}{{\bf Dataset 
					used in testing}} \\ 
			\cmidrule{4-7}
		    &&&
			$\mathcal{D}$ && 
			relabelled-$\mathcal{D}$ \\
			\midrule
			Standard  &&&  43.7\% && 16.2\% \\ 
			Robust  &&& 5.8\% && 65.5\% \\
			\bottomrule
		\end{tabular}
	\end{center}
	\label{tab:adv_next_CIFAR_eval_next}
\end{table}

\subsection{Generalization to CIFAR-10.1}
\label{app:cifar101}
\citet{recht2018cifar10} have constructed an unseen but
distribution-shifted test set for CIFAR-10. They show that for many
previously proposed models, accuracy on the CIFAR-10.1 test set can be
predicted as a linear function of performance on the CIFAR-10 test set. 

As a sanity check (and a safeguard against any potential adaptive
overfitting to the test set via hyperparameters, historical test set reuse,
etc.) we note that the classifiers trained on $\ddet$ and $\drand$ achieve
$44\%$ and $55\%$ generalization on the CIFAR-10.1 test set, respectively.
This demonstrates non-trivial generalization, and actually perform better
than the linear fit would predict (given their accuracies on the CIFAR-10
test set).

\clearpage
\subsection{Omitted Results for Restricted ImageNet}
\label{app:imagenet}

\begin{figure}[h!]
    \begin{center}
	\includegraphics[width=0.5\textwidth]{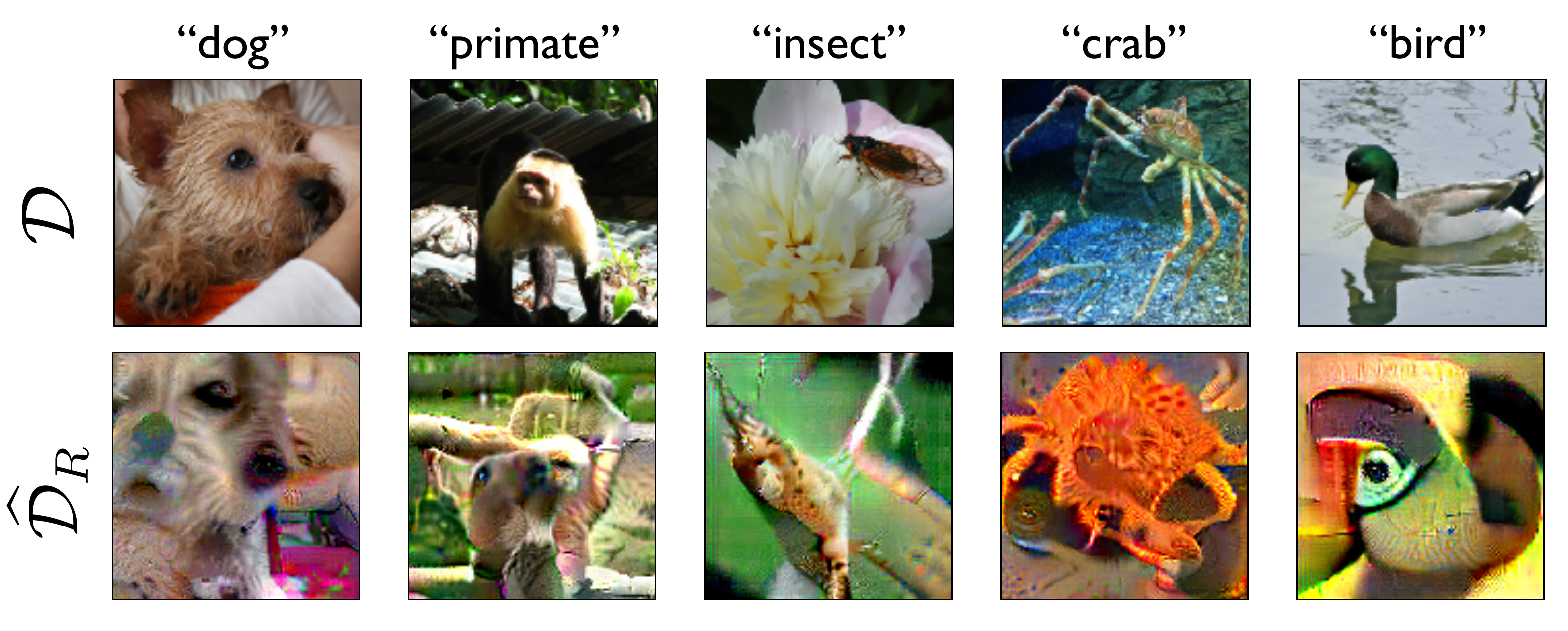}
    \end{center}
    \caption{Repeating the experiments shown in
	Figure~\ref{fig:distill_results} for the Restricted ImageNet
    dataset. Sample images from the resulting dataset.}
\end{figure}
\begin{figure}[h!]
    \begin{center}
	\includegraphics[width=0.5\textwidth]{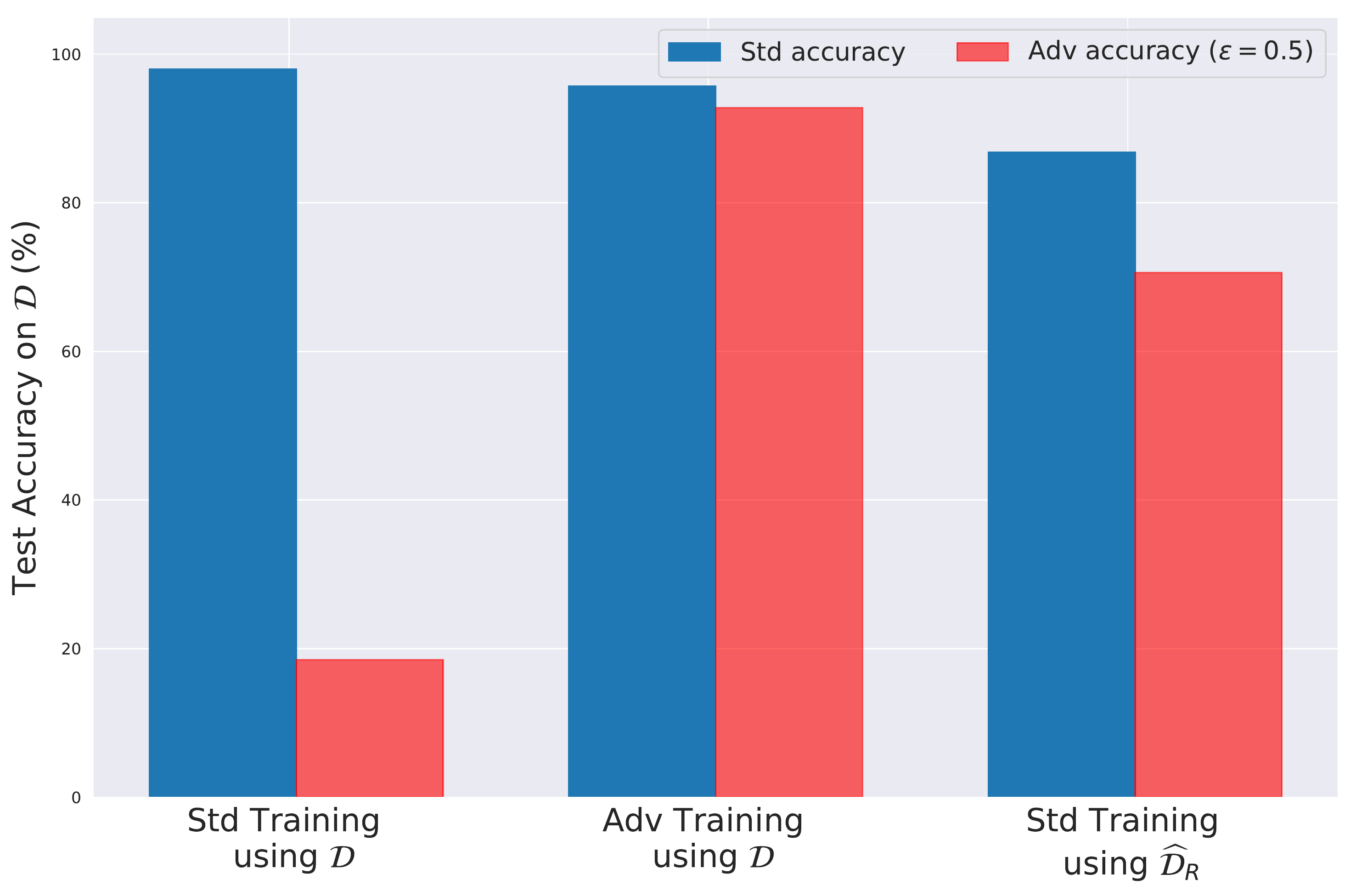}
    \caption{Repeating the experiments shown in
	Figure~\ref{fig:distill_results} for the Restricted ImageNet
	dataset. Standard and robust accuracy of models trained on these
    datasets.}
    \label{fig:distill_imagenet}
    \end{center}
\end{figure}

\clearpage
\subsection{Targeted Transferability}
\label{app:targeted_transfer}
\begin{figure}[h!]
    \begin{center}
	\includegraphics[width=0.5\textwidth]{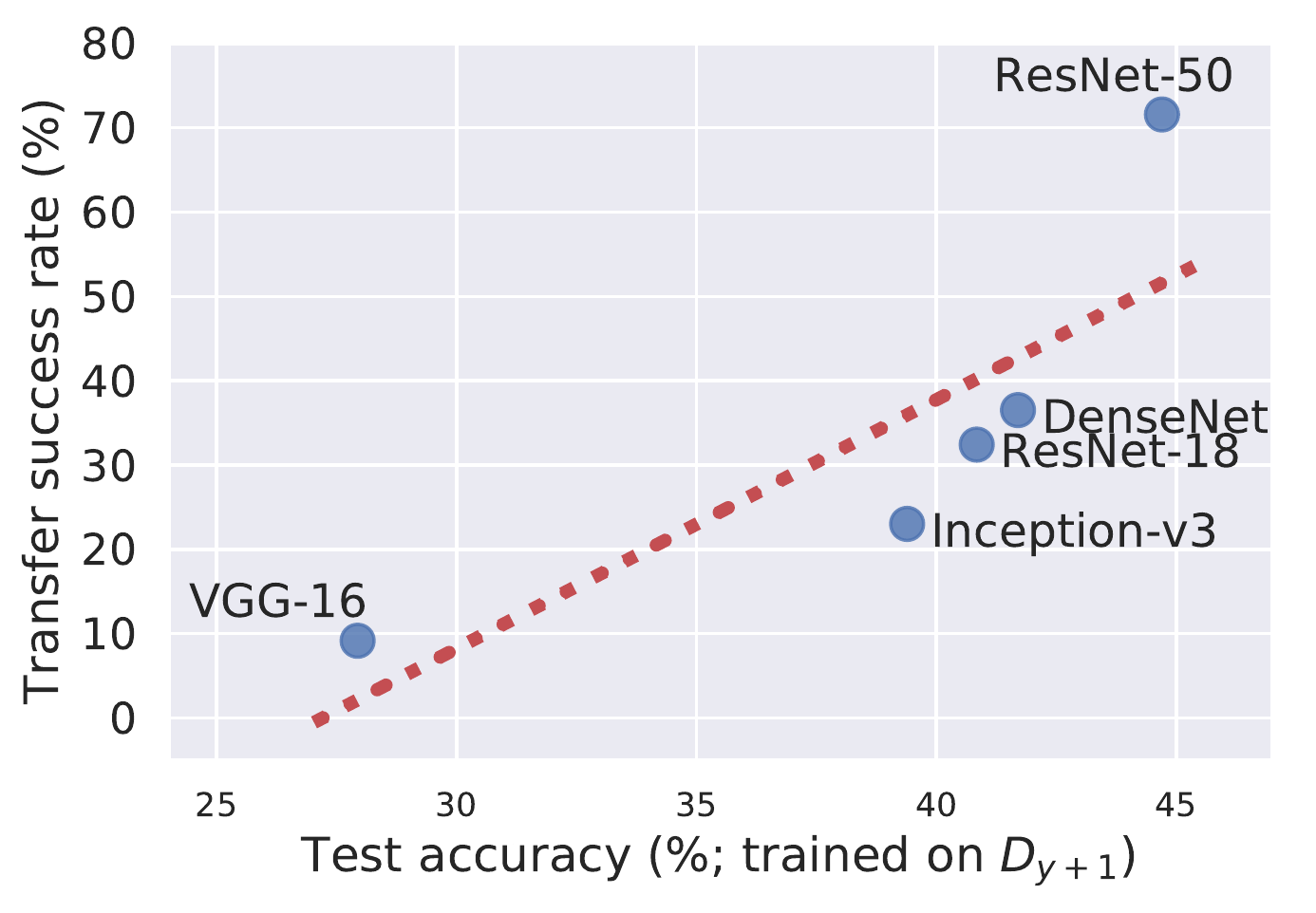}
    \end{center}
    \caption{Transfer rate of {\em targeted} adversarial examples (measured
	in terms of attack success rate, not just misclassification) from a
	ResNet-50 to different architectures alongside test set performance
	of these architecture when trained on the dataset generated in
	Section~\ref{sec:distill_nrf}. Architectures more susceptible to
	transfer attacks also performed better on the standard test set
	supporting our hypothesis that adversarial transferability arises
	from utilizing similar {\em non-robust features}.}
\end{figure}

\subsection{Robustness vs. Accuracy}
\label{subsec:robustnessacc}

\begin{figure}[h!]
    \begin{center}
	\includegraphics[width=0.6\textwidth]{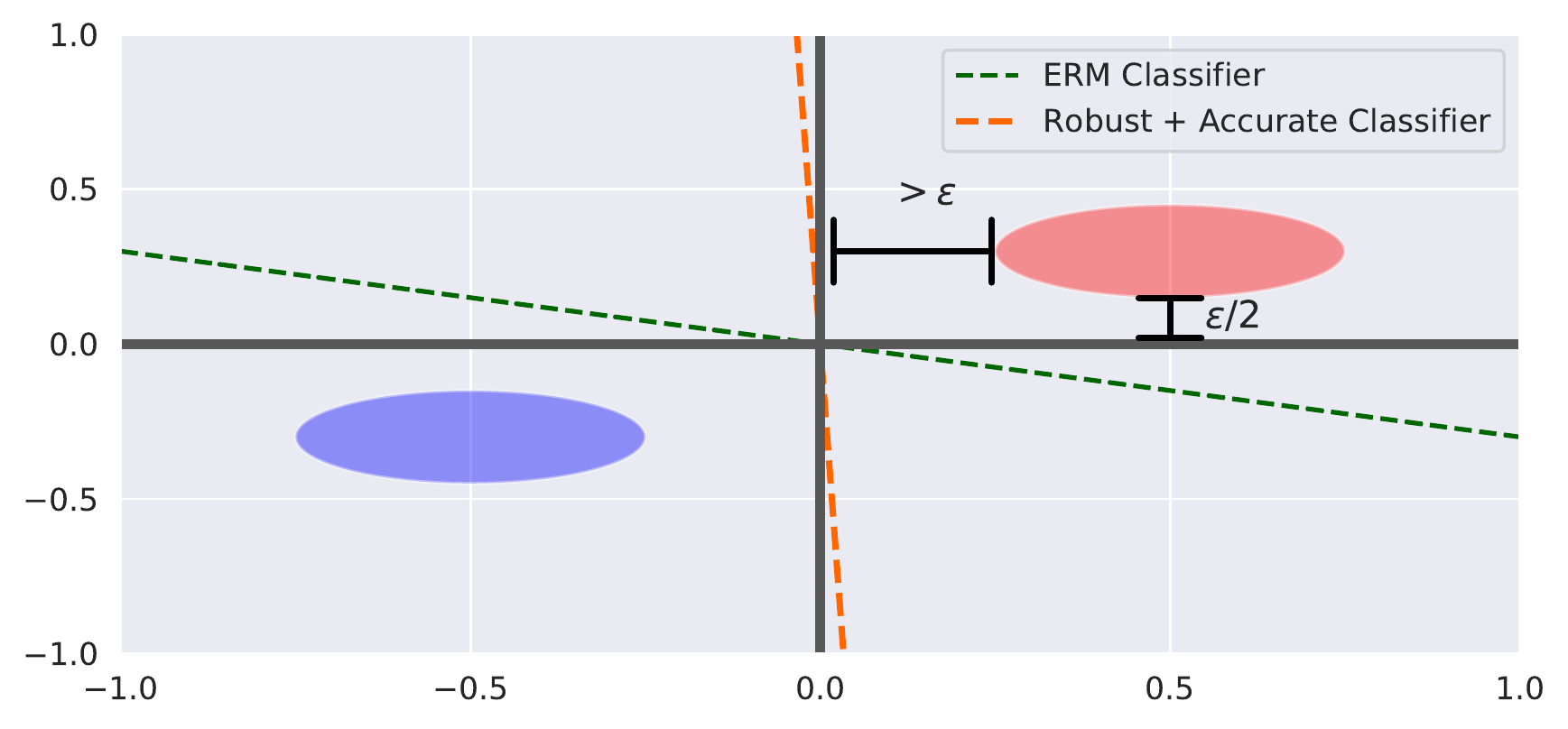}
	\caption{An example where adversarial vulnerability can arise from
	ERM training on any standard loss function due to non-robust features (the
    green line shows the ERM-learned decision boundary). There exists,
however, a classifier that is both perfectly robust {\em and} accurate,
resulting from robust training, which forces the classifier to ignore the
$x_2$ feature despite its predictiveness.}
    \label{fig:robustness_v_acc}
    \end{center}
\end{figure}

\clearpage
\section{Gaussian MLE under Adversarial Perturbation}
\label{app:proof}
In this section, we develop a framework for studying non-robust features by
studying the problem of {\em maximum likelihood classification} between two
Gaussian distributions. We first recall the setup of the problem, then
present the main theorems from Section~\ref{sec:theory}. First we build
the techniques necessary for their proofs.

\subsection{Setup}
\label{sec:setup}

We consider the setup where a learner receives labeled samples from two
distributions, $\mathcal{N}(\bm{\mu}_*, \bm{\Sigma}_*)$, and
$\mathcal{N}(-\bm{\mu}_*, \bm{\Sigma}_*)$. The learner's goal is to be able to
classify new samples as being drawn from $\mathcal{D}_1$ or $\mathcal{D}_2$ according
to a maximum likelihood (MLE) rule.

A simple coupling argument demonstrates that this problem can actually be reduced to
learning the parameters $\widehat{\mu}$, $\widehat{\Sigma}$ of a single Gaussian
$\mathcal{N}(-\bm{\mu}_*, \bm{\Sigma}_*)$, and then employing a linear classifier
with weight $\widehat{\bm{\Sigma}}^{-1}\widehat{\bm{\mu}}$. In the standard setting,
maximum likelihoods estimation learns the true parameters, $\mu_*$ and $\Sigma_*$,
and thus the learned classification rule is $C(x) = \mathbb{1}\{x^\top
\bm{\Sigma}^{-1}\bm{\mu} > 0\}$.

In this work, we consider the problem of {\em adversarially robust} maximum
likelihood estimation. In particular, rather than simply being asked to classify
samples, the learner will be asked to classify {\em adversarially perturbed} samples
$x+\delta$, where $\delta \in \Delta$ is chosen to maximize the loss of the learner.
Our goal is to derive the parameters $\bm{\mu}, \bm{\Sigma}$ corresponding to an
adversarially robust maximum likelihood estimate of the parameters of
$\mathcal{N}(\bm{\mu}_*, \bm{\Sigma}_*)$. Note that since we have access to
$\bm{\Sigma}_*$ (indeed, the learner can just run non-robust MLE to get access), we
work in the space where $\bm{\Sigma}^*$ is a diagonal matrix, and we restrict the
learned covariance $\bm{\Sigma}$ to the set of diagonal matrices.

\paragraph{Notation.} We denote the parameters of the sampled Gaussian by $\bm{\mu}_*
\in \mathbb{R}^d$, and $\bm{\Sigma}_* \in \{\text{diag}(\bm{u}) | \bm{u} \in
\mathbb{R}^d\}$. We use $\sigma_{min}(X)$ to represent the smallest eigenvalue of a
square matrix $X$, and $\ell(\cdot;x)$ to represent the Gaussian negative
log-likelihood for a single sample $x$. For convenience, we often use $\bm{v} = x -
\bm{\mu}$, and $R = \|\bm{\mu}_*\|$. We also define the $\dia$ operator to represent
the vectorization of the diagonal of a matrix. In particular, for a matrix $X \in
\mathbb{R}^{d\times d}$, we have that $X_{\dia} = v \in \mathbb{R}^d$ if $v_i =
X_{ii}$.

\subsection{Outline and Key Results}
\label{sec:outline}
We focus on the case where $\Delta = \mathcal{B}_2(\epsilon)$ for some
$\epsilon > 0$, i.e. the $\ell_2$ ball, corresponding to the following minimax
problem:
\begin{equation}
    \min_{\bm{\mu}, \bm{\Sigma}} \E\left[\max_{\delta: \|\delta\| = \eps}
    \ell(\bm{\mu},\bm{\Sigma}; x+\delta)\right]
    \label{eq:origprop}
\end{equation}
We first derive the optimal adversarial perturbation for this setting
(Section~\ref{sec:motivating_example}), and prove Theorem~\ref{thm:0}
(Section~\ref{sec:fixed_lagrange}). We then
propose an alternate problem, in which the adversary picks a linear operator to be
applied to a fixed vector, rather than picking a specific perturbation
vector (Section~\ref{sec:real_objective}). We
argue via Gaussian concentration that the alternate problem is indeed reflective of
the original model (and in particular, the two become equivalent as
$d\rightarrow\infty$). In particular, we propose studying the following in place
of~\eqref{eq:origprop}:
\begin{align}
    \label{eq:realprob}
    &\min_{\bm{\mu}, \bm{\Sigma}} \max_{M\in \mathcal{M}} \E\left[
    \ell(\bm{\mu},\bm{\Sigma}; x+M(x-\bm{\mu}))\right]  \\
    \nonumber
    &\text{where } \mathcal{M} = \left\{M \in \mathbb{R}^{d\times d}:\ M_{ij} = 0\
    \forall\ i \neq j,\ 
	\mathbb{E}_{x\sim\mathcal{N}(\bm{\mu}^*, \bm{\Sigma}^*)}
	\left[
	    \|M\bm{v}\|^2_2
    \right] = \epsilon^2 \right\}.
\end{align}

Our goal is to characterize the behavior of the robustly
learned covariance $\bm{\Sigma}$ in terms of the true covariance matrix
$\bm{\Sigma}_*$ and the perturbation budget $\eps$. The proof is through
Danskin's Theorem, which allows us to use any maximizer of the inner problem $M^*$
in computing the subgradient of the inner minimization. After showing the
applicability of Danskin's Theorem (Section~\ref{sec:danskin_valid}) and
then applying it (Section~\ref{sec:applying_danskin}) to prove our main
results (Section~\ref{sec:main_thms}). Our three main results, which
we prove in the following section, are presented below. 

First, we consider a simplified version of~\eqref{eq:origprop},
in which the adversary solves a maximization with a fixed Lagrangian
penalty, rather than a hard $\ell_2$ constraint. In this setting, we show
that the loss contributed by the adversary corresponds to a misalignment
between the data metric (the Mahalanobis distance, induced by
$\Sigma^{-1}$), and the $\ell_2$ metric:
\vulnerability*

\noindent We then return to studying~\eqref{eq:realprob}, where we provide
upper and lower bounds on the learned robust covariance matrix $\bm{\Sigma}$:
\parameters*

\noindent Finally, we show that in the worst case over mean vectors
$\mu_*$, the gradient of the adversarial robust classifier aligns more with
the inter-class vector:
\gradients*

\subsection{Proofs}
In the first section, we have shown that the classification between two Gaussian
distributions with identical covariance matrices centered at $\bm{\mu}^*$ and
$-\bm{\mu}^*$ can in fact be reduced to learning the parameters of a
single one of these distributions. 

Thus, in the standard setting, our goal is to solve the following problem:
\begin{align*}
    \min_{\bm{\mu}, \bm{\Sigma}} \mathbb{E}_{x\sim\mathcal{N}(\bm{\mu}^*,
    \bm{\Sigma}^*)}\left[\ell(\bm{\mu},\bm{\Sigma};x)\right] := 
    \min_{\bm{\mu}, \bm{\Sigma}} \mathbb{E}_{x\sim\mathcal{N}(\bm{\mu}^*,
    \bm{\Sigma}^*)}\left[-\log\left(\mathcal{N}(\bm{\mu},\bm{\Sigma};x)\right)\right].
\end{align*}

Note that in this setting, one can simply find differentiate $\ell$ with respect to
both $\bm{\mu}$ and $\bm{\Sigma}$, and obtain closed forms for both (indeed, these
closed forms are, unsurprisingly, $\bm{\mu}^*$ and $\bm{\Sigma}^*$). Here, we
consider the existence of a {\em malicious adversary} who is allowed to
perturb each sample point $x$ by some $\delta$. The goal of the
adversary is to {\em maximize} the same loss that the learner is minimizing. 

\subsubsection{Motivating example: \texorpdfstring{$\ell_2$}{l-2}-constrained adversary}
\label{sec:motivating_example}

We first consider, as a motivating example, an $\ell_2$-constrained adversary. That
is, the adversary is allowed to perturb each sampled point by $\delta: \|\delta\|_2 =
\eps$. In this case, the minimax problem being solved is the following:
\begin{align}
    \label{eq:origagain}
    \min_{\bm{\mu}, \bm{\Sigma}} \mathbb{E}_{x\sim\mathcal{N}(\bm{\mu}^*,
    \bm{\Sigma}^*)}
    \left[
	\max_{\|\delta\| = \eps} \ell(\bm{\mu},\bm{\Sigma};x+\delta)
    \right].
\end{align}
The following Lemma captures the optimal behaviour of the adversary:

\begin{lemma}
    \label{lemma:optimaldelta}
    In the minimax problem captured in~\eqref{eq:origagain} (and earlier
    in~\eqref{eq:origprop}), the optimal adversarial perturbation $\delta^*$ is given
    by
    \begin{equation}
	\delta^* = \left(\lambda\bm{I} -
	\bm{\Sigma}^{-1}\right)^{-1}\bm{\Sigma}^{-1}\bm{v} = 
	\left(\lambda\bm{\Sigma} - \bm{I}\right)^{-1}\bm{v},
    \end{equation}
    where $\bm{v} = x - \bm{\mu}$, and $\lambda$ is set such that $\|\delta^*\|_2 = \eps$.
\end{lemma}
\begin{proof}
In this context, we can solve the inner maximization problem with Lagrange
multipliers. In the following we write $\Delta = \mathcal{B}_2(\eps)$ for brevity,
and discard terms not containing $\delta$ as well as constant factors freely:
\begin{align}
    \nonumber
    \arg\max_{\delta\in\Delta}\ \ell(\bm{\mu}, \bm{\Sigma};x+\delta) -
    &= \arg\max_{\delta\in\Delta}
    \left(x+\delta-\bm{\mu}\right)^\top\bm{\Sigma}^{-1}
    \left(x+\delta-\bm{\mu}\right) \\
    \nonumber
    &= \arg\max_{\delta\in\Delta}\ 
    (x-\bm{\mu})^\top\bm{\Sigma}^{-1}(x-\bm{\mu}) + 
    2\delta^\top\bm{\Sigma}^{-1}(x-\bm{\mu}) +  
    \delta^\top\bm{\Sigma}^{-1}\delta \\
    &= \arg\max_{\delta\in\Delta}\ \delta^\top\bm{\Sigma}^{-1}(x-\bm{\mu}) +
    \frac{1}{2}\delta^\top\bm{\Sigma}^{-1}\delta.
    \label{eq:ready_to_lagrange}
\end{align}
Now we can solve~\eqref{eq:ready_to_lagrange} using the aforementioned Lagrange
multipliers. In particular, note that the maximum of~\eqref{eq:ready_to_lagrange} is
attained at the boundary of the $\ell_2$ ball $\Delta$. Thus, we can solve the
following system of two equations to find $\delta$, rewriting the norm constraint as
$\frac{1}{2}\|\delta\|_2^2 = \frac{1}{2}\eps^2$:
\begin{equation}
    \begin{cases}
	\nabla_\delta \left(\delta^\top\bm{\Sigma}^{-1}(x-\bm{\mu}) +
	\frac{1}{2}\delta^\top \bm{\Sigma}^{-1}\delta\right)= \lambda \nabla_\delta
	\left(\|\delta\|_2^2-\eps^2\right) \implies \bm{\Sigma}^{-1}(x-\bm{\mu}) +
	\bm{\Sigma}^{-1}\delta =
	\lambda\delta\\
	\|\delta\|_2^2 = \eps^2.
    \end{cases}
\end{equation}
For clarity, we write $\bm{v} = x - \bm{\mu}$: then, combining the above, we have
that 
\begin{equation}\label{eq:deltastar}
    \delta^* = \left(\lambda\bm{I} -
    \bm{\Sigma}^{-1}\right)^{-1}\bm{\Sigma}^{-1}\bm{v} = 
    \left(\lambda\bm{\Sigma} - \bm{I}\right)^{-1}\bm{v},
\end{equation}
our final result for the maximizer of the inner problem, where $\lambda$ is set
according to the norm constraint.
\end{proof}

\subsubsection{Variant with Fixed Lagrangian (Theorem~\ref{thm:0})}
\label{sec:fixed_lagrange}
To simplify the analysis of Theorem~\ref{thm:0}, we consider a version
of~\eqref{eq:origagain} with a fixed Lagrangian penalty, rather than a norm
constraint:
$$ \max \ell(x + \delta; y\cdot\mu, \Sigma) - C\cdot \|\delta\|_2.$$
Note then, that by Lemma~\ref{lemma:optimaldelta}, the optimal
perturbation $\delta^*$ is given by
$$\delta^* = \left(C\Sigma - \bm{I}\right)^{-1}.$$
We now proceed to the proof of Theorem~\ref{thm:0}.
\vulnerability*
\begin{proof}
    We begin by expanding the Gaussian negative log-likelihood for the
    relaxed problem:
    \begin{align*}
	\mathcal{L}_{adv}(\Theta) - \mathcal{L}(\Theta) &=
	 \E\left[
	     2\cdot \bm{v}^\top\left(C\cdot \Sigma -
	     \bm{I}\right)^{-\top}\Sigma^{-1} \bm{v} + \bm{v}^\top
		 \left( C\cdot \Sigma - \bm{I}\right)^{-\top}
		 \Sigma^{-1} 
		 \left( C\cdot\Sigma - \bm{I}\right)^{-1}
		\bm{v}
	 \right] \\
	 &=
	 \E\left[
	     2\cdot \bm{v}^\top\left(C\cdot\Sigma \Sigma -
	     \Sigma\right)^{-1} \bm{v} + \bm{v}^\top
		 \left( C\cdot\Sigma  - \bm{I}\right)^{-\top}
		 \Sigma^{-1} 
		 \left( C\cdot\Sigma  - \bm{I}\right)^{-1}
		\bm{v}
	    \right]
    \end{align*}
    Recall that we are considering the vulnerability at the MLE parameters
    $\mu^*$ and $\Sigma^*$:
    \begin{align*}
	\mathcal{L}_{adv}(\Theta) - \mathcal{L}(\Theta) &=
	\mathbb{E}_{\bm{v}\sim\mathcal{N}(0, I)}\left[
	    2\cdot \bm{v}^\top\Sigma_*^{1/2}\left(C\cdot\Sigma_*^2 -
	    \Sigma_*\right)^{-1} \Sigma_*^{1/2}\bm{v} \right. \\ 
	    &\qquad \left. +\  \bm{v}^\top
	     \Sigma_*^{1/2} \left( C\cdot\Sigma_*  - \bm{I}\right)^{-\top}
		 \Sigma_*^{-1} 
		 \left( C\cdot\Sigma_*  - \bm{I}\right)^{-1}
		 \Sigma_*^{1/2}\bm{v}
	\right]\\
	&= \mathbb{E}_{\bm{v}\sim\mathcal{N}(0, I)}\left[
	2\cdot \bm{v}^\top\left(C\cdot\Sigma_* -
	\bm{I}\right)^{-1} \bm{v} + \bm{v}^\top
	     \Sigma_*^{1/2} \left(C^2\Sigma_*^3  
		 - 2C\cdot \Sigma_*^2 + \Sigma_* \right)^{-1}
		 \Sigma_*^{1/2}\bm{v}
	\right]\\
	&= \mathbb{E}_{\bm{v}\sim\mathcal{N}(0, I)}\left[
	2\cdot \bm{v}^\top\left(C\cdot\Sigma_* -
	\bm{I}\right)^{-1} \bm{v} + \bm{v}^\top
	\left(C\cdot\Sigma_* - \bm{I} \right)^{-2}\bm{v}
	 \right]\\
	 &= \mathbb{E}_{\bm{v}\sim\mathcal{N}(0, I)}\left[- \|v\|_2^2 +
	 \bm{v}^\top\bm{I}\bm{v} +
	2\cdot \bm{v}^\top\left(C\cdot\Sigma_* -
	\bm{I}\right)^{-1} \bm{v} + \bm{v}^\top
	\left(C\cdot\Sigma_* - \bm{I} \right)^{-2}\bm{v}
	 \right]\\
	 &= \mathbb{E}_{\bm{v}\sim\mathcal{N}(0, I)}\left[- \|v\|_2^2 +
	 \bm{v}^\top\left(\bm{I} + \left(
		 C\cdot\Sigma_* - \bm{I}\right)^{-1}\right)^{2} \bm{v}
	 \right] \\
	 &= \tr\left[\left(\bm{I} + \left(
		 C\cdot\Sigma_* - \bm{I}\right)^{-1}\right)^2
	 \right] - d
     \end{align*}
     This shows the first part of the theorem. It remains to show that for
     a fixed $k = \tr(\Sigma_*)$, the adversarial risk is minimized by
     $\Sigma_* = \frac{k}{d}\bm{I}$:
     \begin{align*}
	 \min_{\Sigma_*}\ \mathcal{L}_{adv}(\Theta) - \mathcal{L}(\Theta) &=
	 \min_{\Sigma_*}\ \tr\left[\left(\bm{I} + \left(
		 C\cdot\Sigma_* - \bm{I}\right)^{-1}\right)^2
	 \right] \\
	 &= \min_{\{\sigma_i\}}\ \sum_{i=1}^d 
	 \left(1 + \frac{1}{C\cdot\sigma_i - 1}\right)^2,
     \end{align*}
     where $\{\sigma_i\}$ are the eigenvalues of $\Sigma_*$. Now, we have
     that $\sum \sigma_i = k$ by assumption, so by optimality conditions,
     we have that $\Sigma_*$ minimizes the above if $\nabla_{\{\sigma_i\}}
     \propto \vec{1}$, i.e. if $\nabla_{\sigma_i} = \nabla_{\sigma_j}$ for
     all $i, j$. Now,
     \begin{align*}
	 \nabla_{\sigma_i} &= 
	 -2\cdot \left(1 + \frac{1}{C\cdot\sigma_i - 1}\right)\cdot
	 \frac{C}{\left(C\cdot\sigma_i - 1\right)^2} \\ 
	 &= -2\cdot \frac{C^2\cdot \sigma_i}{(C\cdot \sigma_i - 1)^3}.
     \end{align*}
     Then, by solving analytically, we find that 
     \begin{align*}
	 -2\cdot \frac{C^2\cdot \sigma_i}{(C\cdot \sigma_i - 1)^3} = 
	 -2\cdot \frac{C^2\cdot \sigma_j}{(C\cdot \sigma_j - 1)^3}
     \end{align*}
     admits only one real solution, $\sigma_i = \sigma_j$. Thus, $\Sigma_*
     \propto \bm{I}$. Scaling to satisfy the trace constraint yields
     $\Sigma_* = \frac{k}{d}\bm{I}$, which concludes the proof.
\end{proof}

\subsubsection{Real objective}
\label{sec:real_objective}
Our motivating example (Section~\ref{sec:motivating_example}) demonstrates that the
optimal perturbation for the adversary in the $\ell_2$-constrained case is actually a
linear function of $\bm{v}$, and in particular, that the optimal perturbation can be
expressed as $D\bm{v}$ for a diagonal matrix $D$. Note, however, that the problem
posed in~\eqref{eq:origagain} is not actually a minimax problem, due to the presence of
the expectation between the outer minimization and the inner maximization. Motivated
by this and~\eqref{eq:deltastar}, we define the following robust problem:
\begin{align}
    \label{eq:new}
    &\min_{\bm{\mu}, \bm{\Sigma}} \max_{M \in \mathcal{M}} 
    \mathbb{E}_{x\sim\mathcal{N}(\bm{\mu}^*,
    \bm{\Sigma}^*)}
    \left[
	\ell(\bm{\mu},\bm{\Sigma};x+M\bm{v})
    \right], \\
    \nonumber
    &\text{where } \mathcal{M} = \left\{M \in \mathbb{R}^{d\times d}:\ M_{ij} = 0\
    \forall\ i \neq j,\ 
	\mathbb{E}_{x\sim\mathcal{N}(\bm{\mu}^*, \bm{\Sigma}^*)}
	\left[
	    \|M\bm{v}\|^2_2
    \right] = \epsilon^2 \right\}.
\end{align}
First, note that this objective is slightly different from that
of~\eqref{eq:origagain}.
In the motivating example, $\delta$ is constrained to {\em always} have $\eps$-norm,
and thus is normalizer on a per-sample basis inside of the expectation. In contrast,
here the classifier is concerned with being robust to perturbations that are linear
in $\bm{v}$, and of $\eps^2$ squared norm {\em in expectation}. \\

\noindent Note, however, that via the result of~\citet{laurent2000adaptive}
showing strong concentration for the norms of Gaussian random variables, in high
dimensions this bound on expectation has a corresponding high-probability bound on
the norm. In particular, this implies that as $d \rightarrow \infty$, $\|M\bm{v}\|_2
= \eps$ almost surely, and thus the problem becomes identical to that
of~\eqref{eq:origagain}. We now derive the optimal $M$ for a given $(\mu, \Sigma)$:
\begin{lemma}
    Consider the minimax problem described by~\eqref{eq:new}, i.e.
\begin{align*}
    &\min_{\bm{\mu}, \bm{\Sigma}} \max_{M \in \mathcal{M}} 
    \mathbb{E}_{x\sim\mathcal{N}(\bm{\mu}^*,
    \bm{\Sigma}^*)}
    \left[
	\ell(\bm{\mu},\bm{\Sigma};x+M\bm{v})
    \right].
\end{align*}
Then, the optimal action $M^*$ of the inner maximization problem is given by
\begin{equation}
    \label{eq:mstar}
    M = \left(\lambda \bm{\Sigma} - \bm{I}\right)^{-1},
\end{equation}
where again $\lambda$ is set so that $M \in \mathcal{M}$.
\end{lemma}
\begin{proof}
We accomplish this in a similar fashion to
what was done for $\delta^*$, using Lagrange multipliers:
\begin{align*}
    \nabla_M \E
    \left[ \bm{v}^\top M\bm{\Sigma}^{-1}\bm{v} +
    \frac{1}{2}\bm{v}^\top M \bm{\Sigma}^{-1}M \bm{v} \right] &= 
    \lambda \nabla_M \E
    \left[\|M\bm{v}\|_2^2-\eps^2\right]  \\
    \E\left[\bm{\Sigma}^{-1}\bm{v}\bm{v}^\top + \bm{\Sigma}^{-1}M\bm{v}\bm{v}^\top\right] 
    &= \E\left[\lambda M\bm{v}\bm{v}^\top\right] \\
    \bm{\Sigma}^{-1}\bm{\Sigma}^* + \bm{\Sigma}^{-1}M\bm{\Sigma}^* 
    &= \lambda M \bm{\Sigma}^* \\
    M &= \left(\lambda \bm{\Sigma} - \bm{I}\right)^{-1},
\end{align*}
where $\lambda$ is a constant depending on $\bm{\Sigma}$ and $\bm{\mu}$ enforcing the
expected squared-norm constraint. 
\end{proof}

\noindent Indeed, note that the optimal $M$ for the adversary
takes a near-identical form to the optimal $\delta$~\eqref{eq:deltastar}, with the
exception that $\lambda$ is not sample-dependent but rather varies only with the
parameters.

\subsubsection{Danskin's Theorem} 
\label{sec:danskin_valid}
The main tool in proving our key results is
Danskin's Theorem~\citep{danskin1967theory}, a powerful theorem from
minimax optimization which contains the following key result:
\begin{theorem}[Danskin's Theorem]
    \label{thm:danskin}
    Suppose $\phi(x,z): \mathbb{R} \times Z \rightarrow \mathbb{R}$ is a continuous
    function of two arguments, where $Z \subset \mathbb{R}^m$ is compact. Define 
    $f(x) = \max_{z \in Z} \phi(x,z)$.  Then, if for every $z \in Z$, $\phi(x,z)$ is convex and
    differentiable in $x$, and $\frac{\partial \phi}{\partial x}$ is continuous:

The subdifferential of $f(x)$ is given by
$$\partial f(x) = \mathrm{conv} \left\{ \frac{\partial \phi(x,z)}{\partial x} : z \in
Z_0(x) \right\},$$
where $\mathrm{conv}(\cdot)$ represents the convex hull operation, and $Z_0$ is the
set of maximizers defined as
$$Z_0(x) = \left\{ \overline{z} : \phi(x,\overline{z}) = \max_{z \in Z}
\phi(x,z)\right\}.$$
\end{theorem}

In short, given a minimax problem of the form $\min_x\max_{y\in C} f(x, y)$ where $C$
is a compact set, if $f(\cdot, y)$ is convex for all values of $y$, then rather than
compute the gradient of $g(x) := \max_{y\in C} f(x, y)$, we can simply find a
maximizer $y^*$ for the current parameter $x$; Theorem~\ref{thm:danskin} ensures that
$\nabla_x f(x, y^*) \in \partial_x g(x)$. Note that $\mathcal{M}$ is trivially
compact (by the Heine-Borel theorem), and differentiability/continuity follow rather
straightforwardly from our reparameterization (c.f.~\eqref{eq:grad}), and so it
remains to show that the outer minimization is convex for any fixed $M$.

\paragraph{Convexity of the outer minimization.} 
Note that even in the standard case (i.e. non-adversarial), the Gaussian negative
log-likelihood is not convex with respect to $(\bm{\mu}, \bm{\Sigma})$. Thus, rather
than proving convexity of this function directly, we employ the
parameterization used by~\cite{daskalakis2019efficient}: in particular, we
write the problem in terms of $\bm{T} = \bm{\Sigma}^{-1}$ and $\bm{m} =
\bm{\Sigma}^{-1}\bm{\mu}$. Under this
parameterization, we show that the robust problem is convex for any fixed $M$.

\begin{lemma}
    Under the aforementioned parameterization 
of $\bm{T} = \bm{\Sigma}^{-1}$ and $\bm{m} = \bm{\Sigma}^{-1}\bm{\mu}$, the following
``Gaussian robust negative log-likelihood'' is convex:
$$\E\left[\ell(\bm{m}, \bm{T};x+M\bm{v})\right].$$ 
\end{lemma}
\begin{proof}
To prove this, we show that the likelihood is convex even with respect to a single
sample $x$; the result follows, since a convex combination of convex functions
remains convex. We begin by
looking at the likelihood of a single sample $x \sim \mathcal{N}(\bm{\mu}_*,
\bm{\Sigma}_*)$:
\begin{align*}
    \mathcal{L}(\bm{\mu}, \bm{\Sigma}; x + M(x-\bm{\mu})) &= 
\frac{1}{\sqrt{(2 \pi)^{k}|\boldsymbol{\Sigma}|}} 
\exp \left(-\frac{1}{2}(x-\bm{\mu})^{\top}
(I+M)^2\bm{\Sigma}^{-1}
(x-\bm{\mu})\right) \\[1em]
&= \frac{
    \frac{1}{\sqrt{(2 \pi)^{k}|\bm{\Sigma}|}} 
    \exp \left(-\frac{1}{2}(x-\bm{\mu})^{\top}
    (I+M)^2\bm{\Sigma}^{-1}
    (x - \bm{\mu})\right) 
}{
    {\displaystyle\int} \frac{1}{\sqrt{(2 \pi)^{k}|\bm{(I+M)^{-2}\Sigma}|}} 
    \exp \left(-\frac{1}{2}(x - \bm{\mu})^{\top}
    (I+M)^2\bm{\Sigma}^{-1}
    (x - \bm{\mu})\right) 
} \\[1em]
&= \frac{
|I+M|^{-1}
    \exp \left(-\frac{1}{2}x^\top (I+M)^2\bm{\Sigma}^{-1} x +
    \bm{\mu}^{\top} (I+M)^2\bm{\Sigma}^{-1} x\right) 
}{
    {\displaystyle\int}
    \exp \left(-\frac{1}{2}x^\top (I+M)^2\bm{\Sigma}^{-1} x +
    \bm{\mu}^{\top} (I+M)^2\bm{\Sigma}^{-1} x\right) 
}
\end{align*}
In terms of the aforementioned $\bm{T}$ and $\bm{m}$, and for convenience defining $A = (I+M)^2$:
\begin{align}
    \nonumber
    \ell(x) &= |A|^{-1/2} + \left(\frac{1}{2}x^\top A\bm{T}x - \bm{m}^\top A x\right)
    - \log\left(
	{\displaystyle\int}
	\exp\left(\frac{1}{2}x^\top A\bm{T}x - \bm{m}^\top A x\right)
    \right) \\
    \nonumber
    \nabla \ell(x) &= \tvgrad
    - \frac{
	{\displaystyle\int} 
	\tvgrad
	\exp\left(\frac{1}{2}x^\top A\bm{T}x - \bm{m}^\top A x\right)
	}{
	{\displaystyle\int}
	\exp\left(\frac{1}{2}x^\top A\bm{T}x - \bm{m}^\top A x\right)
	} \\
	\label{eq:grad}
	&=  \tvgrad - \mathbb{E}_{z\sim \mathcal{N}(\bm{T}^{-1}\bm{m},
	(\bm{AT})^{-1})}
	 \stack{\frac{1}{2} (Azz^\top)_{\dia}}{- Az}.
\end{align}
From here, following an identical argument
to~\cite{daskalakis2019efficient} Equation (3.7), we find that
\begin{align*}
\bm{H}_{\ell} = \mathrm{Cov}_{z \sim \mathcal{N}\left(\bm{T}^{-1} \bm{m},
(A\bm{T})^{-1}\right)}\left[\left( \begin{array}{c}{\left(-\frac{1}{2} Az
z^{T}\right)_{\dia}} \\ {z}\end{array}\right), \left(
\begin{array}{c}{\left(-\frac{1}{2} Az z^{T}\right)_{\dia}} \\
{z}\end{array}\right)\right] \succcurlyeq \bm{0},
\end{align*}

i.e. that the log-likelihood is indeed convex with respect to
$\stack{\bm{T}}{\bm{m}}$, as desired.
\end{proof}

\subsubsection{Applying Danskin's Theorem} 
\label{sec:applying_danskin}
The previous two parts show that we can
indeed apply Danskin's theorem to the outer minimization, and in particular that the
gradient of $f$ at $M = M^*$ is in the subdifferential of the
outer minimization
problem. We proceed by writing out this gradient explicitly, and then setting it to
zero (note that since we have shown $f$ is convex for all choices of perturbation, we
can use the fact that a convex function is globally minimized $\iff$ its
subgradient contains zero). We continue from above, plugging in~\eqref{eq:mstar}
for $M$ and using~\eqref{eq:grad} to write the gradients of $\ell$ with respect to
$\bm{T}$ and $\bm{m}$. 
\begin{align}
    \nonumber
    0 = \nabla_{\stack{\bm{T}}{\bm{m}}} \ell  &= 
    \E\left[
    \tvgrad - \mathbb{E}_{z\sim \mathcal{N}(\bm{T}^{-1}\bm{m},
	(\bm{AT})^{-1})}
	\stack{\frac{1}{2} (Azz^\top)_{\dia}}{- Az}
    \right] \\
    \nonumber
    &=
    \E\tvgrad - 
    \mathbb{E}_{z\sim \mathcal{N}(\bm{T}^{-1}\bm{m},
	(\bm{AT})^{-1})}
	\stack{\frac{1}{2} (Azz^\top)_{\dia}}{- Az}\\
    \nonumber
    &=
	\stack{\frac{1}{2} (A\bm{\Sigma}_*)_{\dia}}{-A\bm{\mu}_*} - 
    \mathbb{E}_{z\sim \mathcal{N}(\bm{T}^{-1}\bm{m},
	(\bm{AT})^{-1})}
	\stack{\frac{1}{2} (A(A\bm{T})^{-1})_{\dia}}{- A\bm{T}^{-1}\bm{m}}\\
    \nonumber
    &=
	\stack{\frac{1}{2} A\bm{\Sigma}_*}{-A\bm{\mu}_*} - 
	\stack{\frac{1}{2} A(A\bm{T})^{-1}}{- A\bm{T}^{-1}\bm{m}}\\
    &= \stack{\frac{1}{2} A\bm{\Sigma}_* - \frac{1}{2} \bm{T}^{-1}}
	{A\bm{T}^{-1}\bm{m} - A\bm{\mu}_*} 
	\label{eq:expgrad}
\end{align}

Using this fact, we derive an {\em implicit} expression for the robust covariance
matrix $\bm{\Sigma}$. Note that for the sake of brevity, we now use $M$ to denote the
optimal adversarial perturbation (previously defined as $M^*$ in~\eqref{eq:mstar}).
This implicit formulation forms the foundation of the bounds given by our main
results.

\begin{lemma}
    \label{lemma:solution}
    The minimax problem discussed throughout this work admits the following
    (implicit) form of solution:
\begin{align*}
    \bm{\Sigma} &= \frac{1}{\lambda}I +
	\frac{1}{2}\bm{\Sigma}_* + \sqrt{ \frac{1}{\lambda}\bm{\Sigma}_* +
	\frac{1}{4} \bm{\Sigma}_*^2 },
\end{align*}
where $\lambda$ is such that $M \in \mathcal{M}$, and is thus dependent on
$\bm{\Sigma}$.
\end{lemma}
\begin{proof}
Rewriting~\eqref{eq:expgrad} in the standard parameterization (with respect to
$\bm{\mu}, \bm{\Sigma}$) and re-expanding $A = (I+M)^2$ yields:
\begin{align*}
    0 = \nabla_{\stack{\bm{T}}{\bm{m}}} \ell  =
    \stack{\frac{1}{2} (I+M)^2\bm{\Sigma}_* - \frac{1}{2} \bm{\Sigma}}
    {(I+M)^2 \bm{\mu}- (I+M)^2 \bm{\mu}_*} 
\end{align*} 

Now, note that the equations involving $\bm{\mu}$ and $\bm{\Sigma}$ are completely
independent, and thus can be solved separately. In terms of $\bm{\mu}$, the relevant
system of equations is $A\bm{\mu} - A\bm{\mu}_* = 0$, where multiplying by the
inverse $A$ gives that 
\begin{equation}
    \label{eq:truemu}\bm{\mu} = \bm{\mu}_*.
\end{equation}
This tells us that the mean learned
via $\ell_2$-robust maximum likelihood estimation is precisely the true mean of the
distribution.

Now, in the same way, we set out to find $\bm{\Sigma}$ by solving the relevant system
of equations:
\begin{align}
    \label{eq:quad}
    \bm{\Sigma}_*^{-1} &= \bm{\Sigma}^{-1}(M+I)^2.
\end{align}

Now, we make use of the Woodbury Matrix Identity in order to write
$(I+M)$ as
\begin{equation*}
    I + (\lambda\bm{\Sigma} - I)^{-1} = I + \left(-I -
	\left(\frac{1}{\lambda}\bm{\Sigma}^{-1} -
    I\right)^{-1}\right) = -\left(\frac{1}{\lambda}\bm{\Sigma}^{-1} - I\right)^{-1}.
\end{equation*}
Thus, we can revisit~\eqref{eq:quad} as follows:
\begin{align*}
    \bm{\Sigma}_*^{-1} &= \bm{\Sigma}^{-1}\left(\frac{1}{\lambda}\bm{\Sigma}^{-1} -
    I\right)^{-2} \\
    \frac{1}{\lambda^2}\bm{\Sigma}_*^{-1}\bm{\Sigma}^{-2} -
    \left(\frac{2}{\lambda}\bm{\Sigma}_*^{-1} + I\right)\bm{\Sigma}^{-1} + 
    \bm{\Sigma}_*^{-1} &= 0 \\
    \frac{1}{\lambda^2}\bm{\Sigma}_*^{-1} -
    \left(\frac{2}{\lambda}\bm{\Sigma}_*^{-1} + I\right)\bm{\Sigma} + 
    \bm{\Sigma}_*^{-1}\bm{\Sigma}^{2} &= 0
\end{align*}

We now apply the quadratic formula to get an implicit expression for $\bm{\Sigma}$
(implicit since technically $\lambda$ depends on $\bm{\Sigma}$):
\begin{align}
    \nonumber
    \bm{\Sigma} &= \left(\frac{2}{\lambda}\bm{\Sigma}_*^{-1} +
    I \pm \sqrt{ \frac{4}{\lambda}\bm{\Sigma}_*^{-1} + I
}\right)\frac{1}{2}\bm{\Sigma}_* \\
	    &= \frac{1}{\lambda}I +
	\frac{1}{2}\bm{\Sigma}_* + \sqrt{ \frac{1}{\lambda}\bm{\Sigma}_* +
	\frac{1}{4} \bm{\Sigma}_*^2 }\label{eq:final}.
\end{align}
This concludes the proof. 
\end{proof}

\subsubsection{Bounding \texorpdfstring{$\lambda$}{\textlambda}} 
\label{sec:bounding_lambda}
We now attempt to characterize the shape of $\lambda$ as a function of $\eps$. 
First, we use the fact that $\mathbb{E}[\|Xv\|^2] = \tr(X^2)$
for standard normally-drawn $v$.  Thus, $\lambda$ is set such that
$\tr(\bm{\Sigma}_*M^2) = \eps$, i.e:
\begin{equation}
    \label{eq:trtr}
    \sum_{i=0} \frac{\bm{\Sigma}^*_{ii}}{(\lambda \bm{\Sigma}_{ii} - 1)^2} = \eps
\end{equation}
Now, consider $\eps^2$ as a function of $\lambda$.
Observe that for $\lambda \geq \frac{1}{\sigma_{min}(\bm{\Sigma})}$, we have that $M$
must be positive semi-definite, and thus $\eps^2$ decays smoothly from $\infty$ (at
$\lambda = \frac{1}{\sigma_{min}})$ to zero (at $\lambda = \infty$). Similarly, for
$\lambda \leq \frac{1}{\sigma_{max}(\bm{\Sigma})}$, $\eps$ decays smoothly as
$\lambda$ {\em decreases}. Note, however, that such values of $\lambda$ would
necessarily make $M$ {\em negative semi-definite}, which would actually {\em help}
the log-likelihood. Thus, we can exclude this case; in particular, for the remainder
of the proofs, we can assume $\lambda \geq \frac{1}{\sigma_{max}(\bm{\Sigma})}$.

Also observe that the zeros of $\eps$ in terms of $\lambda$ are only at $\lambda =
\pm \infty$. Using this, we can show that there exists some $\eps_0$ for which, for
all $\eps < \eps_0$, the only corresponding possible valid value of $\lambda$ 
is where $\lambda \geq \frac{1}{\sigma_{min}}$. This idea is formalized in the
following Lemma.
\begin{lemma}
    \label{lemma:eps0_exists}
    For every $\bm{\Sigma}_*$, there exists some $\eps_0 > 0$ for which, for all
    $\eps \in [0, \eps_0)$ the only admissible value of $\lambda$ is such that
    $\lambda \geq \frac{1}{\sigma_{min}(\bm{\Sigma})}$, and thus such that $M$ is
    positive semi-definite.
\end{lemma}
\begin{proof}
    We prove the existence of such an $\eps_0$ by lower bounding $\eps$ (in terms of
    $\lambda$) for any finite $\lambda > 0$ that does not make $M$ PSD. Providing
    such a lower bound shows that for small enough $\eps$ (in particular, less than this lower bound), the only
    corresponding values of $\lambda$ are as desired in the statement\footnote{Since
	our only goal is existence, we lose many factors from the
    analysis that would give a tighter bound on $\eps_0$.}.

    In particular, if $M$ is not PSD, then there must exist at least one index $k$
    such that $\lambda\bm{\Sigma}_{kk} < 1$, and thus $(\lambda\bm{\Sigma}_{kk} -
    1)^2 \leq 1$ for all $\lambda > 0$. We can thus lower bound~\eqref{eq:trtr}
    as:
\begin{equation}
    \eps = \sum_{i=0} \frac{\bm{\Sigma}^*_{ii}}{(\lambda \bm{\Sigma}_{ii} - 1)^2} \geq 
    \frac{\bm{\Sigma}^*_{kk}}{(\lambda \bm{\Sigma}_{kk} - 1)^2} \geq
    \bm{\Sigma}^*_{kk} \geq \sigma_{min}(\bm{\Sigma}^*) > 0
\end{equation}
By contradiction, it follows that for any $\eps < \sigma_{min}(\bm{\Sigma}_*)^2$, the
only admissible $\lambda$ is such that $M$ is PSD, i.e. according to the statement of
the Lemma.
\end{proof}

In the regime $\eps \in [0, \eps_0)$, note that $\lambda$ is inversely proportional to
$\eps$ (i.e. as $\eps$ grows, $\lambda$ decreases). This allows us to get a
qualitative view of~\eqref{eq:final}: as the allowed perturbation value increases,
the robust covariance $\bm{\Sigma}$ resembles the identity matrix more and more, and
thus assigns more and more variance on initially low-variance features. The
$\sqrt{\Sigma_*}$ term indicates that the robust model also adds uncertainty
proportional to the square root of the initial variance---thus, low-variance features
will have (relatively) more uncertainty in the robust case.
Indeed, our main result actually follows as a (somewhat loose) formalization of this
intuition.

\subsubsection{Proof of main theorems}
\label{sec:main_thms}
First, we give a proof of Theorem~\ref{thm:1}, providing lower and upper bounds on
the learned robust covariance $\bm{\Sigma}$ in the regime $\eps \in [0, \eps_0)$.

\parameters*
\begin{proof}
    We have already shown that $\mu = \mu_*$ in the robust case
    (c.f.~\eqref{eq:truemu}).
    We choose $\eps_0$ to be as described, i.e. the largest $\eps$ for
    which the set $\{\lambda: \tr(\bm{\Sigma}_*^2 M) = \eps, \lambda \geq
    1/\sigma_{\max}(\bm{\Sigma})\}$ has only one element $\lambda$ (which,
    as we argued, must not be less than $1/\sigma_{min}(\bm{\Sigma})$). We
    have argued that such an $\eps_0$ must exist.

    We prove the result by combining our early derivation (in
    particular,~\eqref{eq:quad} and~\eqref{eq:final}) with upper and lower bound on
    $\lambda$, which we can compute based on properties of the trace operator. We
    begin by deriving a lower bound on $\lambda$. By linear algebraic manipulation
    (given in Appendix~\ref{app:lambdabounds}), we get the following bound:
    \begin{equation}
    \lambda \geq \frac{d}{\tr(\bm{\Sigma})}\left(1 +
    \sqrt{\frac{d \cdot \sigma_{min}(\bm{\Sigma}_*)}{\eps}}\right) 
    \label{eq:lambda1}
    \end{equation}

    Now, we can use~\eqref{eq:quad} in order to remove the dependency of $\lambda$ on
    $\bm{\Sigma}$:
    \begin{align*}
	\bm{\Sigma} &= \bm{\Sigma}_*(M + I)^2 \\
	%&= (\bm{\Sigma}_*^{1/2} M + \bm{\Sigma}_*^{1/2})^2 \\
	\tr(\bm{\Sigma}) &= 
	\tr\left[
	    (\bm{\Sigma}_*^{1/2} M + \bm{\Sigma}_*^{1/2})^2 
	\right]\\
	&\leq 
	2\cdot \tr\left[
	    (\bm{\Sigma}_*^{1/2} M)^2 + (\bm{\Sigma}_*^{1/2})^2 
	\right]\\
	&\leq 2\cdot \left(\eps + \tr(\bm{\Sigma}_*)\right).
    \end{align*}
    Applying this to~\eqref{eq:lambda1} yields:
    \begin{align*}
	\lambda &\geq \frac{d/2}{\eps + \tr(\bm{\Sigma}_*)}
	\left(1 + \sqrt{\frac{d \cdot \sigma_{min}(\bm{\Sigma}_*)}{\eps}}\right).
    \end{align*}
    Note that we can simplify this bound significantly by writing $\eps = d\cdot
    \sigma_{min}(\bm{\Sigma}_*)\eps' \leq \tr(\bm{\Sigma}_*)\eps'$, which does not
    affect the result (beyond rescaling the valid regime $(0, \eps_0)$), and gives:
    \begin{align*}
	\lambda &\geq \frac{d/2}{(1+\eps')\tr(\bm{\Sigma}_*)}
	\left(1 + \frac{1}{\sqrt{\eps'}}\right) 
	\geq \frac{d\cdot
    (1+\sqrt{\eps'})}{2\sqrt{\eps'}(1+\eps')\tr(\bm{\Sigma}_*)}
    \end{align*}
    %%%
    Next, we follow a similar methodology (Appendix~\ref{app:lambdabounds}) in order
    to upper bound $\lambda$:
    \begin{equation*}
    \lambda \leq \frac{1}{\sigma_{min}(\bm{\Sigma})}\left(
    \sqrt{\frac{\|\bm{\Sigma}_*\|_F\cdot d}{\eps}} +1\right).
    \end{equation*}
    Note that by~\eqref{eq:quad} and positive semi-definiteness of $M$, it must be
    that $\sigma_{min}(\bm{\Sigma}) \geq \sigma_{min}(\bm{\Sigma}_*)$. Thus, we can
    simplify the previous expression, also substituting $\eps =
    d\cdot\sigma_{min}(\bm{\Sigma}_*)\eps'$:
    $$
	\lambda \leq \frac{1}{\sigma_{min}(\bm{\Sigma}_*)}\left(
	\sqrt{\frac{\|\bm{\Sigma}_*\|_F
    }{\sigma_{min}(\bm{\Sigma}_*)\eps'}}+1\right) = 
    \frac{\|\bm{\Sigma}_*\|_F + \sqrt{\eps\cdot \sigma_{min}(\bm{\Sigma}_*)}}{\sigma_{min}(\bm{\Sigma}_*)^{3/2}\sqrt{\eps}}
    $$
    These bounds can be straightforwardly combined with
    Lemma~\ref{lemma:solution}, which concludes the proof.
\end{proof}

Using this theorem, we can now show Theorem~\ref{thm:2}:
\gradients*
\begin{proof}
    To prove this, we make use of the following Lemmas:
    \begin{lemma}
	For two positive definite matrices $A$ and $B$ with $\kappa(A) >
	\kappa(B)$, we have that $\kappa(A+B) \leq \max\{\kappa(A),
	\kappa(B)\}$.
    \end{lemma}
    \begin{proof}
	We proceed by contradiction:
	\begin{align*}
	    \kappa(A+B) &= \frac{\lambda_{max}(A) +
	\lambda_{max}(B)}{\lambda_{min}(A) + \lambda_{min}(B)} \\ 
	\kappa(A) &= \frac{\lambda_{max}(A)}{\lambda_{min}(A)} \\
	\kappa(A) &\geq \kappa(A+B) \\
	\iff \lambda_{max}(A)\left(\lambda_{min}(A) +
	\lambda_{min}(B)\right) &\geq
	\lambda_{min}(A)\left(\lambda_{max}(A) +
	\lambda_{max}(B)\right)  \\ 
	\iff \lambda_{max}(A)\lambda_{min}(B) &\geq
	\lambda_{min}(A) \lambda_{max}(B) \\
	\iff \frac{\lambda_{max}(A)}{\lambda_{min}(A)} &\geq
	\frac{\lambda_{min}(A)}{\lambda_{max}(B)},
    \end{align*}
    which is false by assumption. This concludes the proof.
    \end{proof}

    \begin{lemma}[Straightforward]
	For a positive definite matrix $A$ and $k > 0$, we have that
	$$\kappa(A+k\cdot I) < \kappa(A)\qquad \kappa(A+k\cdot \sqrt{A})
	\leq \kappa(A).$$
    \end{lemma}

    \begin{lemma}[Angle induced by positive definite matrix;
	folklore]\footnote{A proof can be found in \url{https://bit.ly/2L6jdAT}}
	For a positive definite matrix $A \succ 0$ with condition number
	$\kappa$, we have that
	\begin{equation}
	    \label{eq:cond_num_angle}
	    \min_{x} \frac{x^\top A x}{\|Ax\|_2\cdot \|x\|_2} =
	    \frac{2\sqrt{\kappa}}{1+\kappa}.
	\end{equation}
    \end{lemma}

    These two results can be combined to prove the theorem. First, we show
    that $\kappa(\Sigma) \leq \kappa(\Sigma_*)$:
    \begin{align*}
	\kappa(\Sigma) &= \kappa\left(
	    \frac{1}{\lambda}I + \frac{1}{2}\Sigma_* +
	    \sqrt{\frac{1}{\lambda}\Sigma_* + \frac{1}{4}\Sigma_*^2}
	\right) \\
	&< \max\left\{
	    \kappa\left(\frac{1}{\lambda}I + \frac{1}{2}\Sigma_*\right),
	    \kappa\left(\sqrt{\frac{1}{\lambda}\Sigma_* +
	    \frac{1}{4}\Sigma_*^2}\right)
	\right\} \\
	&< \max\left\{
	    \kappa\left(\Sigma_*\right),
	    \sqrt{\kappa\left(\frac{1}{\lambda}\Sigma_* +
	    \frac{1}{4}\Sigma_*^2\right)}
	\right\} \\
	&= \max\left\{
	    \kappa\left(\Sigma_*\right),
	    \sqrt{\kappa\left(\frac{2}{\lambda}\sqrt{\frac{1}{4}\Sigma_*^2} +
	    \frac{1}{4}\Sigma_*^2\right)}
	\right\} \\
	&\leq \kappa\left(\Sigma_*\right).
    \end{align*}
    Finally, note that~\eqref{eq:cond_num_angle} is a strictly decreasing
    function in $\kappa$, and as such, we have shown the theorem.
\end{proof}

\subsubsection{Bounds for \texorpdfstring{$\lambda$}{\textlambda}}
\label{app:lambdabounds}
\paragraph{Lower bound.}
\begin{align*}
    \eps &= \tr(\bm{\Sigma}_*M^2) \\
    &\geq \sigma_{min}(\bm{\Sigma}_*)\cdot \tr(M^2) 
    &\text{ by the definition of $\tr(\cdot)$}\\
    &\geq \frac{\sigma_{min}(\bm{\Sigma}_*)}{d} \cdot \tr(M)^2 
    &\text{ by Cauchy-Schwarz}\\
    &\geq \frac{\sigma_{min}(\bm{\Sigma}_*)}{d} \cdot
    \left[\tr\left((\lambda\bm{\Sigma} - \bm{I})^{-1}\right)\right]^2 
    &\text{ Expanding $M$~\eqref{eq:mstar}}\\
    &\geq \frac{\sigma_{min}(\bm{\Sigma}_*)}{d} \cdot
    \left[\tr\left(\lambda\bm{\Sigma} - \bm{I}\right)^{-1}\cdot d^2\right]^{2} 
    &\text{ AM-HM inequality}\\
    &\geq d^3\cdot \sigma_{min}(\bm{\Sigma}_*) \cdot
    \left[\lambda\cdot \tr(\bm{\Sigma}) - d\right]^{-2} \\
    \nonumber
    \left[\lambda\cdot \tr(\bm{\Sigma}) - d\right]^2 &\geq
    \frac{d^3\cdot \sigma_{min}(\bm{\Sigma}_*)}{\eps} \\
    \lambda\cdot \tr(\bm{\Sigma}) - d &\geq
    \frac{d^{3/2} \cdot \sqrt{\sigma_{min}(\bm{\Sigma}_*)}}{\sqrt{\eps}} 
    &\text{ since $M$ is PSD} \\
    \lambda &\geq \frac{d}{\tr(\bm{\Sigma})}\left(1 +
    \sqrt{\frac{d \cdot \sigma_{min}(\bm{\Sigma}_*)}{\eps}}\right) 
\end{align*}

\paragraph{Upper bound}
\begin{align*}
    \eps &= \tr(\bm{\Sigma}_*M^2) \\
    &\leq \|\bm{\Sigma}_*\|_F \cdot d\cdot \sigma_{max}(M)^2 \\
    &\leq \|\bm{\Sigma}_*\|_F \cdot d\cdot \sigma_{min}(M)^{-2} \\
    \lambda\cdot\sigma_{min}(\bm{\Sigma}) - 1 &\leq 
    \sqrt{\frac{\|\bm{\Sigma}_*\|_F\cdot d}{\eps}} \\
    \lambda &\leq \frac{1}{\sigma_{min}(\bm{\Sigma})}\left(
    \sqrt{\frac{\|\bm{\Sigma}_*\|_F\cdot d}{\eps}} +1\right).
\end{align*}

\clearpage

\end{document}